\documentclass{article}

\usepackage{microtype}
\usepackage{graphicx}
\usepackage{subfigure}
\usepackage{booktabs} 
\usepackage{enumitem}
\usepackage{amsfonts}
\usepackage{amsmath}
\usepackage{amssymb}
\usepackage{mathtools}
\usepackage{amsthm}
\usepackage{multirow}
\usepackage[dvipsnames]{xcolor}
\usepackage{bbm}
\usepackage{placeins}

\usepackage{xspace}
\newcommand\ie{i.\,e.\xspace}
\newcommand\eg{e.\,g.\xspace}

\theoremstyle{plain}
\newtheorem{theorem}{Theorem}[section]

\newtheorem{lemma}[theorem]{Lemma}
\newtheorem{corollary}[theorem]{Corollary}
\theoremstyle{definition}

\newtheorem{assumption}[theorem]{Assumption}
\theoremstyle{remark}

\usepackage{hyperref}

\newcommand{\ind}{\perp\!\!\!\!\perp} 
\DeclareMathOperator*{\argmax}{arg\,max}
\DeclareMathOperator*{\argmin}{arg\,min}

\newcommand{\dd}{\mathop{}\!\mathrm{d}}

\usepackage{pifont}
\newcommand{\cmark}{\textcolor{ForestGreen}{\ding{51}}}%
\newcommand{\xmark}{\textcolor{BrickRed}{\ding{55}}}%

\newcommand{\longname}{\emph{Causal Transformer}\xspace}
\newcommand{\shortname}{CT\xspace}

\usepackage[accepted]{icml2022}

\icmltitlerunning{Causal Transformer for Estimating Counterfactual Outcomes}

\begin{document}

\twocolumn[
\icmltitle{Causal Transformer for Estimating Counterfactual Outcomes}




\begin{icmlauthorlist}
\icmlauthor{Valentyn Melnychuk}{lmu}
\icmlauthor{Dennis Frauen}{lmu}
\icmlauthor{Stefan Feuerriegel}{lmu}
\end{icmlauthorlist}

\icmlaffiliation{lmu}{LMU Munich, Munich, Germany}

\icmlcorrespondingauthor{Valentyn Melnychuk}{melnychuk@lmu.de}

\icmlkeywords{counterfactual inference, treatment effect estimation, personalized medicine, transformer}

\vskip 0.3in
]



\printAffiliationsAndNotice{}  

\begin{abstract}
Estimating counterfactual outcomes over time from observational data is relevant for many applications (e.g., personalized medicine). Yet, state-of-the-art methods build upon simple long short-term memory~(LSTM) networks, thus rendering inferences for complex, long-range dependencies challenging. In this paper, we develop a novel \longname for estimating counterfactual outcomes over time. Our model is specifically designed to capture complex, long-range dependencies among time-varying confounders. For this, we combine three transformer subnetworks with separate inputs for time-varying covariates, previous treatments, and previous outcomes into a joint network with in-between cross-attentions. We further develop a custom, end-to-end training procedure for our \longname. Specifically, we propose a novel counterfactual domain confusion loss to address confounding bias: it aims to learn adversarial balanced representations, so that they are predictive of the next outcome but non-predictive of the current treatment assignment. We evaluate our \longname based on synthetic and real-world datasets, where it achieves superior performance over current baselines. To the best of our knowledge, this is the first work proposing transformer-based architecture for estimating counterfactual outcomes from longitudinal data.
\end{abstract}

\section{Introduction} 
\label{sec:intro}    

\begin{figure*}[tbp]
    \vskip -0.07in
    \begin{center}
    \centerline{\includegraphics[width=0.9\textwidth]{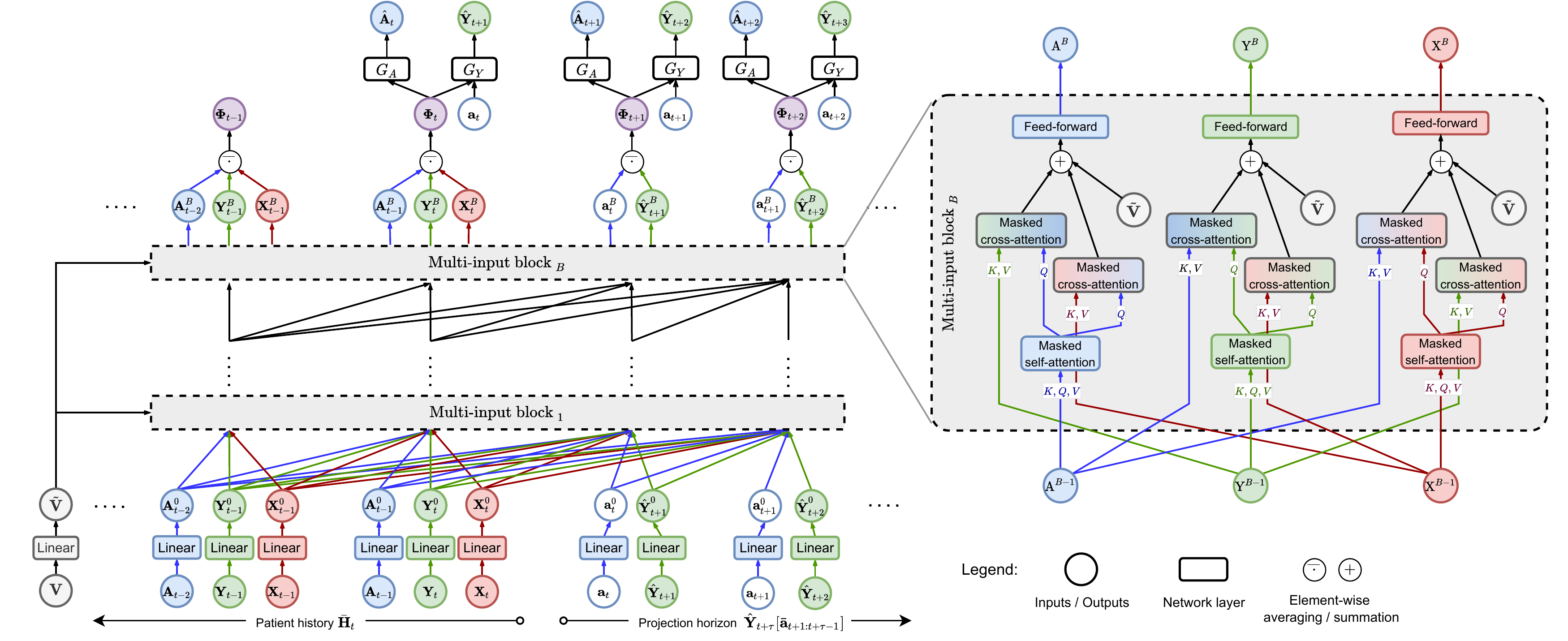}}
    \vskip -0.1in
    \caption{Overview of our \shortname. We distinguish two timelines: time steps $1, \ldots, t$ refer to observational data (patient trajectories) and thus input; time steps $t+1, \ldots, t+\tau$ is the projection horizon and thus output. Three separate transformers are used in parallel for encoding observational data as input: treatments $\mathbf{A}_t$ / treatment interventions $\mathbf{a}_t$ (blue), outcomes $\mathbf{Y}_{t}$ / outcome predictions $\hat{\mathbf{Y}}_{t}$ (green), and time-varying covariates $\mathbf{X}_t$ (red). These are fused via $B$ stacked multi-input blocks. Additional static covariates $\mathbf{V}$ (gray) are fed into all multi-input blocks. Each multi-input block further makes use of cross-attentions. Afterward, the three respective representation for treatments, outcomes, and time-varying covariates are averaged, giving the (balanced) representation $\mathbf{\Phi}_t$ (purple). On top of that are two additional networks $G_Y$ (outcome prediction network) and $G_A$ (treatment classifier network) for learning balanced representations in our CDC loss. Layer normalizations and residual connections are omitted for clarity.
    }
    \label{fig:multi-input-transformer}
    \end{center}
    \vskip -0.3in
\end{figure*} 


Decision-making in medicine requires precise knowledge of individualized health outcomes over time after applying different treatments \cite{huang2012analysis,hill2013assessing}. This then informs the choice of treatment plans and thus ensures effective care personalized to individual patients. Traditionally, the gold standard for estimating the effects of treatments are randomized controlled trials~(RCTs). However, RCTs are costly, often impractical, or even unethical. To address this, there is a growing interest in estimating health outcomes over time from observational data, such as, \eg, electronic health records.


Numerous methods have been proposed for estimating (counterfactual) outcomes from observational data in the static setting \cite{van2006targeted,chipman2010bart,johansson2016learning,curth2021nonparametric,kuzmanovic2022estimating}. Different from that, we focus on longitudinal settings, that is, \emph{over time}. In fact, longitudinal data are nowadays paramount in medical practice. For example, almost all electronic health records (EHRs) nowadays store sequences of medical events over time \cite{allam2021analyzing}. However, estimating counterfactual outcomes over time is challenging. One reason is that counterfactual outcomes are generally never observed. On top of that, directly estimating counterfactual outcomes with traditional machine learning methods in the presence of (time-varying) confounding has a larger generalization error of estimation \cite{alaa2018limits}, or is even biased (in case of multiple-step-ahead prediction) \cite{robins2009estimation,frauen2022estimating}. Instead, tailored methods are needed. 


To estimate counterfactual outcomes over time, state-of-the-art methods make nowadays use of machine learning. Prominent examples are: recurrent marginal structural networks~(RMSNs) \cite{lim2018forecasting}, counterfactual recurrent network~(CRN) \cite{bica2020estimating}, and G-Net \cite{li2021g}. However, these methods build upon simple long short-term memory~(LSTM) networks, because of which their ability to model complex, long-range dependencies in observational data is limited. Long-range dependencies are omnipresent in medical data; \eg, long-term treatment effects have been observed for obesity \cite{latner2000effective}, multiple sclerosis \cite{sormani2015can}, or diabetes \cite{jacobson2013long}. To address this, we develop a \longname~(\shortname) for estimating counterfactual outcomes over time. It is carefully designed to capture complex, long-range dependencies in medical data that are nowadays common in EHRs. 

In this paper, we aim at estimating counterfactual outcomes over time, that is, for one- and multi-step-ahead predictions. For this, we develop a novel \longname~(\shortname). It combines two innovations: (1)~a tailored transformer-based architecture to capture complex, long-range dependencies in the observational data; and (2)~a novel counterfactual domain confusion (CDC) loss for end-to-end training. 

For~(1), we combine three separate transformer subnetworks for processing time-varying covariates, past treatments, and past outcomes, respectively, into a joint network with in-between cross-attentions. Here, each transformer subnetwork is further extended by (i)~masked multi-head self-attention, (ii)~shared trainable relative positional encoding, and (iii)~attentional dropout. 

For~(2), we develop a custom end-to-end training procedure based on our CDC loss. This allows us to solve an adversarial balancing objective in which we balance representations to be (a)~predictive of outcomes and (b)~non-predictive of the current treatment assignment. The latter is crucial to address confounding bias and thus reduces the generalization error of counterfactual prediction. Importantly, this objective is different from previously proposed gradient reversal balancing \cite{ganin2015unsupervised, bica2020estimating}, as it aims to minimize a reversed KL-divergence to build balanced representations.

We demonstrate the effectiveness of our \shortname over state-of-the-art methods using an extensive series of experiments with synthetic and real-world data. Our ablation study (e.g., against a single-subnetwork architecture) shows that neither (1) nor (2) alone are sufficient for learning. Rather, it is crucial to combine our transformer-based architecture based on three subnetworks \emph{and} our novel CDC loss.   


Overall, our \textbf{main contributions} are as follows:\footnote{Code is available online: \url{https://github.com/Valentyn1997/CausalTransformer}}
\vspace{-0.2cm}
\begin{enumerate}[noitemsep]
\item We propose a new end-to-end model for estimating counterfactual outcomes over time: the \longname~(\shortname). To the best of our knowledge, this is the first transformer tailored to causal inference. 
\item We develop a custom training procedure for our \shortname based on a novel counterfactual domain confusion (CDC) loss.
\item We use synthetic and real-world data to demonstrate that our \shortname achieves state-of-the-art performance. We further achieve this both for one- and multi-step-ahead predictions. 
\end{enumerate}
\vspace{-0.3cm}

\section{Related Work} 
\label{sec:related-work}

\paragraph{Estimating counterfactual outcomes in static setting.} 

Extensive literature has focused on estimating counterfactual outcomes (or, analogously, individual treatment effects~[ITE]) in static settings \cite{johansson2016learning,alaa2018bayesian,wager2018estimation,yoon2018ganite,curth2021nonparametric}. Several works have adapted deep learning for that purpose \cite{johansson2016learning,yoon2018ganite}. In the static setting, the input is given by cross-sectional data, and, as such, there are \emph{no} time-varying covariates, treatments, and outcomes. However, we are interested in counterfactual outcome estimation over time. 

\paragraph{Estimating counterfactual outcomes over time.} Methods for estimating time-varying outcomes were originally introduced in epidemiology and make widespread use of simple linear models. Here,  the aim is to estimate average (non-individual) effects of time-varying treatments. Examples of such methods include G-computation, marginal structural models (MSMs), and structural nested models \cite{robins1986new,robins2000marginal,hernan2001marginal,robins2009estimation}. To address the limited expressiveness of linear models, several Bayesian non-parametric methods were proposed \cite{Xu16,schulam2017reliable,soleimani2017treatment}. However, these make strong assumptions regarding the data generation mechanism, and are not designed for multi-dimensional outcomes as well as static covariates. Other methods build upon recurrent neural networks  \cite{qian2021synctwin,berrevoets2021disentangled} but these are restricted to single-time treatments or make stronger assumptions for identifiability, which do not hold for our setting (see Appendix~\ref{app:methods-table}). 

There are several methods that build upon the potential outcomes framework \cite{rubin1978bayesian,robins2009estimation}, and, thus, ensure identifiability by making the same assumptions as we do (see Sec.~\ref{sec:problem-formulation}). Here, state-of-the-art methods are recurrent marginal structural networks~(RMSNs) \cite{lim2018forecasting}, counterfactual recurrent network~(CRN) \cite{bica2020estimating}, and G-Net \cite{li2021g}. These methods address bias due to time-varying confounding in different ways. RMSNs combine two propensity networks and use the predicted inverse probability of treatment weighting~(IPTW) scores for training the prediction networks. CRN uses an adversarial objective to produce a sequence of balanced representations, which are simultaneously predictive of the outcome but non-predictive of the current treatment assignment. G-Net aims to predict both outcomes and time-varying covariates, and then performs G-computation for multiple-step-ahead prediction. All of three aforementioned methods are built on top of one/two-layer LSTM encoder-decoder architectures. Because of that, they are limited in their ability to capture long-range, complex dependencies between time-varying confounders (\ie, time-varying covariates, previous treatments, and previous outcomes). However, such complex data are nowadays widespread in medical practice (\eg, EHRs) \cite{allam2021analyzing}, which may impede the performance of the previous methods for real-world medical data. As a remedy, we develop a \emph{deep} transformer network for counterfactual outcomes estimation over time.

\paragraph{Transformers.} Transformers refer to deep neural networks for sequential data that typically adopt a custom self-attention mechanism \cite{vaswani2017attention}. This makes transformers both flexible and powerful in modeling long-range associative dependencies for sequence-to-sequence tasks. Prominent examples come from natural language processing (e.g., BERT \cite{devlin2019bert}, RoBERTa \cite{liu2019roberta}, and GPT-3 \cite{brown2020language}). Other examples include image understanding tasks \cite{dosovitskiy2020image}, multi-modal tasks (image/video captioning) \cite{liu2021cptr}, math problem solving \cite{schlag2019enhancing}, and time-series forecasting \cite{tang2021probabilistic,zhou2021informer}. However, to the best of our knowledge, no paper has developed transformers specifically for causal inference. This presents our novelty.

\section{Problem Formulation} 
\label{sec:problem-formulation}

We build upon the standard setting for estimating counterfactual outcomes over time as in \cite{robins2009estimation,lim2018forecasting,bica2020estimating,li2021g}. Let $i$ refer to some patient and with health trajectories that span time steps $t = 1, \dots, T^{(i)}$. For each time step $t$ and each patient $i$, we have the following: $d_x$ time-varying covariates $\mathbf{X}_{t}^{(i)} \in \mathbb{R}^{d_x}$; $d_a$ categorical treatments $\mathbf{A}_{t}^{(i)} \in \{a_1, \dots, a_{d_a}\}$; and $d_y$ outcomes $\mathbf{Y}_{t}^{(i)} \in \mathbb{R}^{d_y}$. For example, data from critical care units of COVID-19 patients would involve blood pressure and heart rate as time-varying covariates, ventilation as treatment, and respiratory frequency as outcome. Treatments are modeled as categorical variables as this relates to the question of whether to apply a treatment or not, and is thus consistent with prior works \cite{lim2018forecasting,bica2020estimating,li2021g}. Further, we record static covariates describing a patient $\mathbf{V}^{(i)}$ (\eg, gender, age, or other risk factors). For notation, we omit patient index $(i)$ unless needed.  

For learning, we have access to i.i.d. observational data $\mathcal{D} = \big\{ \{\mathbf{x}_{t}^{(i)}, \mathbf{a}_{t}^{(i)}, \mathbf{y}_{t}^{(i)}\}_{t=1}^{T^{(i)}} \cup \mathbf{v}^{(i)} \big\}_{i=1}^N$. In clinical settings, such data are nowadays widely available in form of EHRs \cite{allam2021analyzing}. Here, we summarize the patient trajectory by $\bar{\mathbf{H}}_{t} = \{ \bar{\mathbf{X}}_{t}, \bar{\mathbf{A}}_{t-1}, \bar{\mathbf{Y}}_{t}, \mathbf{V} \}$, where $\bar{\mathbf{X}}_{t} = (\mathbf{X}_1, \dots, \mathbf{X}_t)$, $\bar{\mathbf{Y}}_{t} = (\mathbf{Y}_1, \dots, \mathbf{Y}_t)$, and $\bar{\mathbf{A}}_{t-1} = (\mathbf{A}_1, \dots, \mathbf{A}_{t-1})$.

We build upon the potential outcomes framework \cite{neyman1923application,rubin1978bayesian} and its extension to time-varying treatments and outcomes \cite{robins2009estimation}. Let $\tau \geq 1$ denote projection horizon for a $\tau$-step-ahead prediction. Further,
let $\bar{\mathbf{a}} _{t: t+\tau-1} = (\mathbf{a}_t, \ldots, \mathbf{a}_{t + \tau - 1})$ 
denote a given (non-random) treatment intervention. Then, we are interested in the potential outcomes, $\mathbf{Y}_{t + \tau}[\bar{\mathbf{a}}_{t: t+\tau-1}]$, under the treatment intervention. However, the potential outcomes for a specific treatment intervention are typically never observed for a patient but must be estimated. Formally, the potential counterfactual outcomes over time are identifiable from factual observational data $\mathcal{D}$ under three standard assumptions: (1)~consistency, (2)~sequential ignorability, and (3)~sequential overlap (see Appendix~\ref{app:assumptions} for details). 


Our task is thus to estimate future counterfactual outcomes $\mathbf{Y}_{t + \tau}$, after applying a treatment intervention $\bar{\mathbf{a}}_{t: t+\tau-1}$ for a given patient history $\bar{\mathbf{H}}_{t}$. Formally, we aim to estimate:
\begin{equation}
    \label{eq:estimand}
     \mathbb{E} \big( \mathbf{Y}_{t + \tau}[\bar{\mathbf{a}}_{t: t+\tau-1}] \;\mid\; \bar{\mathbf{H}}_{t} \big) .
\end{equation} 
To do so, we learn a function $g(\tau, \bar{\mathbf{a}}_{t: t+\tau-1}, \bar{\mathbf{H}}_{t})$. Simply estimating $g(\cdot)$ with traditional machine learning is biased \cite{robins2009estimation}. For example, one reason is that treatment interventions not only influence outcomes but also future covariates. To address this, we develop a tailored model for estimation.

\section{Causal Transformer}

\paragraph{Input.} 

Our \longname~(\shortname) is a single multi-input architecture, which combines three separate transformer subnetworks. Each subnetwork processes a different sequence as input: (i)~past time-varying covariates $\bar{\mathbf{X}}_{t}$; (ii)~past outcomes $\bar{\mathbf{Y}}_{t}$; and (iii)~past treatments before intervention $\bar{\mathbf{A}}_{t-1}$. Since we aim at estimating the counterfactual outcome after treatment intervention, we further input the future treatment assignment that a medical practitioners wants to intervene on. Also, we autoregressively feed predictions of outcomes  $\bar{\hat{\mathbf{Y}}}_{t+1:t+\tau-1}$, starting at the intervention time step (prediction origin). Thus, we concatenate two treatment sequences $\bar{\mathbf{A}}_{t-1} \cup \bar{\mathbf{a}}_{t: t+\tau-1}$, and two outcome sequences $\bar{\mathbf{Y}}_{t} \cup \bar{\hat{\mathbf{Y}}}_{t+1:t+\tau-1}$ for input. Additionally, (iv)~the static covariates $\mathbf{V}$ are fed into all subnetworks.

\subsection{Model architecture}

Our \shortname yields a sequence of treatment-invariant (balanced) \emph{representations} $\bar{\mathbf{\Phi}}_{t+\tau-1} = (\mathbf{\Phi}_1, \dots, \mathbf{\Phi}_{t+\tau-1})$. To do so, we stack $B$ identical \emph{transformer blocks}. The first transformer block receives the three input sequences. The $B$-th transformer block outputs a sequence of representations $\bar{\mathbf{\Phi}}_{t+\tau-1}$. The architecture is shown in Fig.~\ref{fig:multi-input-transformer}.

\paragraph{Transformer blocks.} 

Let $b = 1, \ldots, B$ index the different transformer blocks. Each transformer block receives three parallel sequences of hidden states as input (for each of the input sequences). For time step $t$, we denote the respective hidden state by $\mathbf{A}^b_t$ or $\mathbf{a}^b_t$; $\mathbf{Y}^b_t$ or $\hat{\mathbf{Y}}^b_t$; and $\mathbf{X}^b_t$. We denote size of the hidden states by $d_h$. Further, each transformer block receives a representation vector of the static covariates $\tilde{\mathbf{V}}$ as additional input.  

For the first transformer block ($b=1$), we use linearly-transformed time-series as input:
\begin{align}
    \begin{split}
        & \mathbf{A}_t^0, \mathbf{a}_t^0 = \operatorname{Linear}_A(\mathbf{A}_{t}, \mathbf{a}_{t}), \quad \,\, \mathbf{X}_t^0 = \operatorname{Linear}_X(\mathbf{X}_{t}), \\
        & \mathbf{Y}_t^0, \hat{\mathbf{Y}}_t^0 = \operatorname{Linear}_Y(\mathbf{Y}_{t}, \hat{\mathbf{Y}}_t), \quad \,\, \tilde{\mathbf{V}} = \operatorname{Linear}_V(\mathbf{V}),
    \end{split}
\end{align}
where parameters of fully-connected linear layers are shared for all time steps. All blocks $ \ge 2 $ use the output sequence of the previous block $b-1$ as inputs. For notation, we denote sequences of hidden states after block $b$ by three tensors $\mathrm{A}^b = \big(\bar{\mathbf{A}}^b_{t-1} \cup \bar{\mathbf{a}}^b_{t: t+\tau-1} \big)^\top,  \mathrm{X}^b = \big(\bar{\mathbf{X}}_{t}^b\big)^\top$, and $\mathrm{Y}^b = \big(\bar{\mathbf{Y}}_{t}^b \cup \bar{\hat{\mathbf{Y}}}_{t+1:t+\tau-1}^b \big), ^\top$.

Following \cite{dong2021attention,lu2021pretrained}, each transformer block combines a (i)~multi-head self-/cross-attention, (ii)~feed-forward layer, and (iii)~layer normalization. Details are in Appendix~\ref{app:CT-block}.  

\underline{(i) Multi-head self-/cross-attention} uses a scaled dot-product attention with several parallel attention heads. Each attention head requires a 3-tuple of keys, queries, and values, \ie, $K, Q, V \in \mathbb{R}^{T \times d_{qkv}}$, respectively. These are obtained from a sequence of hidden states  $\mathrm{H}^b = \big(\mathbf{h}_1^b, \dots, \mathbf{h}_t^b\big)^\top \in \mathbb{R}^{T \times d_h}$ ($\mathrm{H}^b$ is one of $\mathrm{A}^b$, $\mathrm{X}^b$ or $\mathrm{Y}^b$, depending on the subnetwork). Formally, we compute
\begin{align}
        & \operatorname{Attn}^{(i)}(Q^{(i)}, K^{(i)}, V^{(i)}) =  \operatorname{softmax}\Big(\frac{Q^{(i)}K^{(i)}{}^\top}{\sqrt{d_{qkv}}}\Big) V^{(i)} , \label{eq:attention}
    \\
        & Q^{(i)} = Q^{(i)}(\mathrm{H}^b) = \mathrm{H}^b \, W_Q^{(i)} + \mathbf{1} b_Q^{(i)}{}^\top , \\
        & K^{(i)} = K^{(i)}(\mathrm{H}^b) = \mathrm{H}^b \, W_K^{(i)} + \mathbf{1} b^{(i)}_K{}^\top , \\ 
        & V^{(i)} = V^{(i)}(\mathrm{H}^b) = \mathrm{H}^b \, W_V^{(i)} + \mathbf{1} b_V^{(i)}{}^\top ,
\end{align}
where $W_Q^{(i)}, W_K^{(i)}, W_V^{(i)} \in \mathbb{R}^{d_h \times d_{qkv}}$ and $b_Q^{(i)}$, $b_Q^{(i)}$, $b_V^{(i)} \in \mathbb{R}^{d_{qkv}}$ are parameters of a single attention head $i$, where $\operatorname{softmax}(\cdot)$ operates separately on each row, and where $\mathbf{1} \in \mathbb{R}^{d_{qkv}}$ is a vector of ones. We set the dimensionality of keys and queries to $d_{qkv} = d_{h} / n_h$, where $n_h$ is the number of heads.

The output of a multi-head attention is a concatenation of the different heads, \ie, 
\begin{equation}
    \operatorname{MHA}(Q, K, V) = \operatorname{Concat}(\operatorname{Attn}^{(1)}, \dots, \operatorname{Attn}^{(n_h)}) .
\end{equation}
Here, we simplified the original multi-head attention in \cite{vaswani2017attention} by omitting the final output projection layer after concatenation to reduce risk of overfitting.

In our \shortname, self-attention uses the sequence of hidden states from the same transformer subnetwork to infer keys, queries, and values, while cross-attention uses the sequence of hidden states of the other two transformer subnetworks as keys and values. We use multiple cross-attentions to exchange the information between parallel hidden states.\footnote{Different variants of combining multiple-input information with self- and cross-attentions were already studied in the context of multi-source translation, e.g., in \cite{libovicky2018input}. Our implementation is closest to parallel attention combination.} These are placed on top of the self-attention layers (see subdiagram in Fig.~\ref{fig:multi-input-transformer}). We add the representation vector of static covariates, $\tilde{\mathbf{V}}$ when pooling different cross-attention outputs. We mask hidden states for self- and cross-attentions by setting the attention logits in Eq.~\eqref{eq:attention} to $-\infty$. This ensures that information flows only from the current input to future hidden states (and not the other way around).

\underline{(ii) Feed-forward layer} ($\operatorname{FF}$) with ReLU activation is applied time-step-wise to the sequence of hidden states, \ie,
\begin{equation*}
    \operatorname{FF}(\mathbf{h}_t) = \operatorname{Linear} \big( \operatorname{ReLU}
     (\operatorname{Linear}(\mathbf{h}_t))\big),
\end{equation*}
where fully-connected linear layers are followed by dropout.

\underline{(iii) Layer normalization} ($\operatorname{LN}$) \cite{lei2016layer} and residual connections are added after each self- and cross-attention. We compute the layer normalization via
\begin{equation}
    \operatorname{LN}(\mathbf{h}_t) = \frac{\gamma}{\sigma} \odot (\mathbf{h}_t - \mu) + \beta ,
\end{equation}
\begin{equation}
    \mu = \frac{1}{d_h} \sum_{j=1}^{d_h} (\mathbf{h}_t)_j, \quad \sigma = \sqrt{\frac{1}{d_h} \sum_{j=1}^{d_h} \big((\mathbf{h}_t)_j - \mu \big)^2} ,
\end{equation}
where $\gamma, \beta \in \mathbb{R}^{d_h}$ are scale and shift parameters and where $\odot$ is an element-wise product. 

\textbf{Balanced representations.} The (balanced) representations are then constructed via average pooling over three (or two) parallel hidden states of the $B$-th transformer block. Thereby, we use a fully-connected linear layer and an exponential linear unit (ELU) non-linearity; \ie,
\begin{align}
\nonumber
    & \mathbf{\tilde{\Phi}}_i = 
    \begin{cases}
        \frac{1}{3}(\mathbf{A}_{i-1}^{B} + \mathbf{X}_i^{B} + \mathbf{Y}_i^{B}), & i \in \{1, \dots, t\} , \\
        \frac{1}{2}(\mathbf{a}_{i-1}^{B} + \hat{\mathbf{Y}}_i^{B}), & i \in \{t+1, \dots, t + \tau - 1\} ,
    \end{cases}  \\
    & \mathbf{\Phi}_t = \operatorname{ELU}(\operatorname{Linear}(\mathbf{\tilde{\Phi}}_t)) \label{eq:output-repr}
\end{align}
where fully-connected linear layer is followed by dropout, $\mathbf{\Phi}_t \in \mathbb{R}^{d_r}$ and $d_r$ is the dimensionality of the balanced representation.  

\subsection{Positional encoding} 

In order to preserve information about the order of hidden states, we make use of position encoding~(PE). This is especially relevant for clinical practice as it allows us to distinguish sequences such as, \eg, $\langle$treatment~A $\mapsto$ side-effect~S $\mapsto$ treatment~B$\rangle$ from $\langle$treatment~A $\mapsto$ treatment~B $\mapsto$ side-effect~S$\rangle$. 

We model information about relative positions in the input at time steps $j$  and $i$ with $0 \le j \le i \le t$ by a set of vectors $a^V_{ij}, a^K_{ij} \in \mathbb{R}^{d_{qkv}}$ \cite{shaw2018self}. Specifically, they are shaped in the form of Toeplitz matrices
\begin{align}
   & a^V_{ij} = w^V_{\operatorname{clip}(j-i, l_{\text{max}})}, \qquad a^K_{ij} = w^K_{\operatorname{clip}(j-i, l_{\text{max}})}, \\
   & \operatorname{clip}(x, l_{\text{max}}) = \max\{ -l_{\text{max}}, \min\{ l_{\text{max}}, x \}\}
\end{align}
with trainable weights $w^K_l, w^V_l \in \mathbb{R}^{d_{qkv}}$, for $l \in \{-l_{\text{max}}, \dots, 0\}$, and where $l_{\text{max}}$ is the maximum distinguishable distance in the relative PE. The above formalization ensures that we obtain \emph{relative} encodings, that is, our \shortname considers the distance between past or current position $j$ and current position $i$, but not the actual location. Furthermore, the current position $i$ attends only to past information or itself, and, thus, we never use $a^V_{ij}$ and $a^K_{ij}$ where $i < j$. As a result, there are only $(l_{\text{max}} + 1) \times d_{qkv}$ parameters to estimate.

We then use the relative PE to modify the self-attention operation (Eq.~\eqref{eq:attention}). Formally, we compute the attention scores via (indices of heads are dropped for clarity)
\begin{align}
    & (\operatorname{Attn}(Q, K, V))_i = \sum_{j=1}^t \alpha_{ij}(V_j + a_{ij}^V) , \\
    & \alpha_{ij} = \operatorname{softmax}_j \left(\frac{Q_i^\top (K_j + a_{ij}^K)}{\sqrt{d_{qkv}}} \right) , \label{eq:attn-relative-enc}
\end{align}
with attention scores $\alpha_{ij}$ and where $K_j$, $V_j$, and $Q_i$ are columns of corresponding matrices and where $\operatorname{softmax}_j$ operates with respect to index $j$. Cross-attention with PE is defined in an analogous way. In our \shortname, the attention scores are shared across all the heads and blocks, as well as the three different subnetworks.

In our \shortname, we use relative positional encodings \cite{shaw2018self} that are incorporated in every self- and cross-attention. This is different from the original transformer \cite{vaswani2017attention}, which used absolute positional encodings with fixed weights for the initial hidden states of the first transformer block (see Appendix~\ref{app:abs-pe} for details). However, relative PE is regarded as more robust and, further, suited for patient trajectories where the order of treatments and diagnoses is informative \cite{allam2021analyzing}, but not the absolute time step. Additionally, it allows for better generalization to unseen sequence lengths: for the ranges beyond the maximal distinguishable distance $l_{\text{max}}$, \shortname stops to distinguish the precise relative location of states and considers everything as distant past information. In line with this, our experiments later also confirm relative PE to be superior over absolute PE.

\begin{figure*}[tbp]
    \vskip -0.05in
    \centering
    \hfill
    \subfigure[One-step-ahead prediction]{\includegraphics[width=0.3\textwidth]{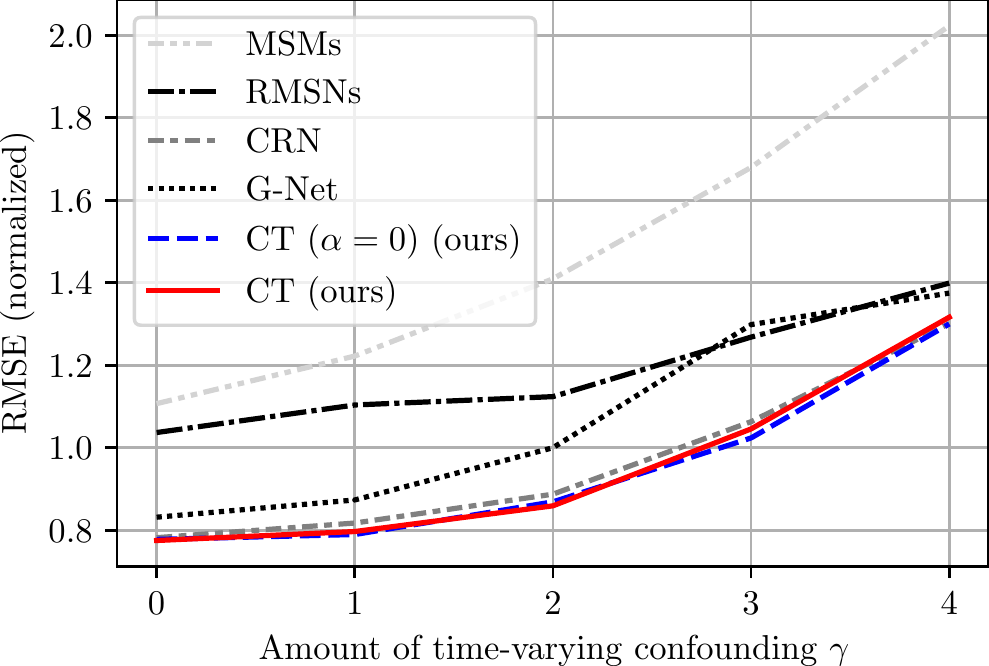}}\label{fig:results-tg-sim-one-step}
    \hfill
    \subfigure[$\tau$-step-ahead prediction (single sliding treatment).]{\includegraphics[width=0.3\textwidth]{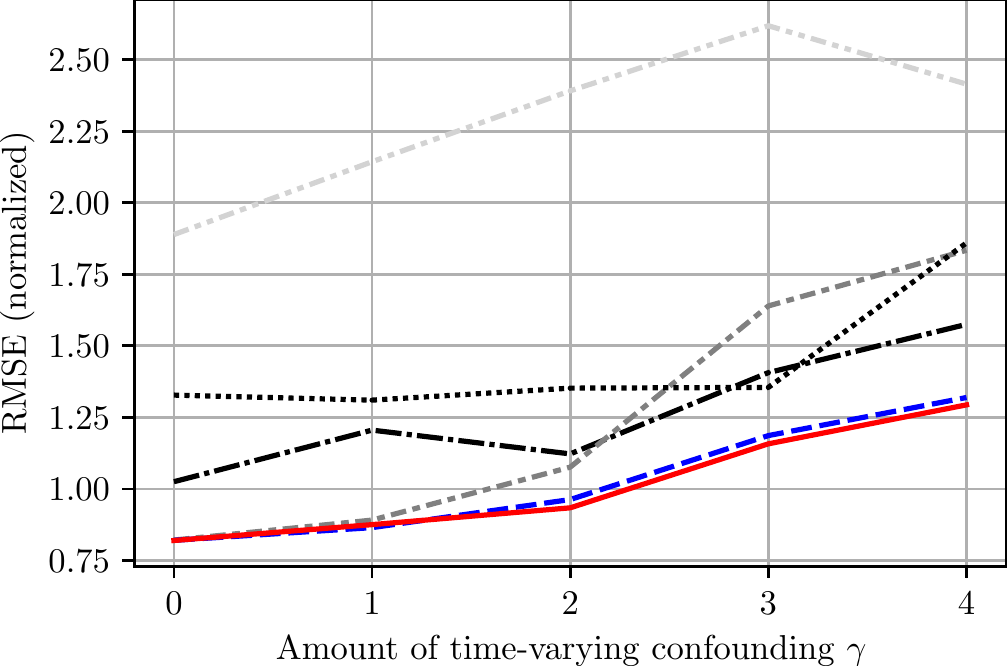}}\label{fig:results-tg-sim-six-step-timing}
    \hfill
    \subfigure[$\tau$-step-ahead prediction (random trajectories)]  {\includegraphics[width=0.3\textwidth]{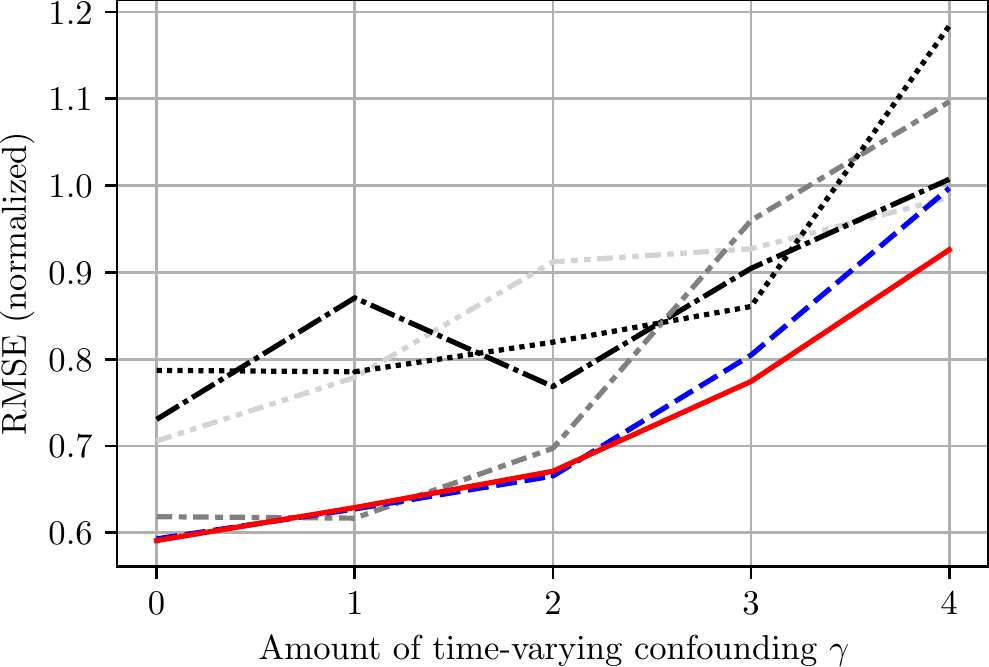}}\label{fig:results-tg-sim-six-step-rand}
    \hfill
    
    \vspace{-0.35cm}
    \caption{Results for fully-synthetic data based on tumor growth simulator (lower values are better). Shown is the mean performance averaged over five runs with different seeds. Here: $\tau = 6$.}
    \label{fig:results-tg-sim}
    \vskip -0.2in
\end{figure*}

\begin{table*}[tbp]
    \caption{Results for semi-synthetic data for $\tau$-step-ahead prediction based on real-world medical data (MIMIC-III). Shown: RMSE as mean $\pm$ standard deviation over five runs. Here: random trajectory setting. MSMs struggle for long prediction horizons with values $>$ 10.0 (due to linear modeling of IPTW scores).}
    \label{tab:ss-sim-all}
    \begin{center}
        \scriptsize
        \addtolength{\tabcolsep}{-1.4pt}  
\begin{tabular}{l|c|ccccccccc}
\toprule
{} &            $\tau = 1$ &            $\tau = 2$ &            $\tau = 3$ &            $\tau = 4$ &            $\tau = 5$ &            $\tau = 6$ &            $\tau = 7$ &            $\tau = 8$ &            $\tau = 9$ &           $\tau = 10$ \\
\midrule
MSMs \cite{robins2000marginal} &           0.37 ± 0.01 &           0.57 ± 0.03 &           0.74 ± 0.06 &           0.88 ± 0.03 &           1.14 ± 0.10 &           1.95 ± 1.48 &           3.44 ± 4.57 & $>$ 10.0 & $>$ 10.0 & $>$ 10.0 \\
RMSNs \cite{lim2018forecasting}                 &           0.24 ± 0.01 &           0.47 ± 0.01 &           0.60 ± 0.01 &           0.70 ± 0.02 &           0.78 ± 0.04 &           0.84 ± 0.05 &           0.89 ± 0.06 &           0.94 ± 0.08 &           0.97 ± 0.09 &           1.00 ± 0.11 \\
CRN \cite{bica2020estimating}                     &           0.30 ± 0.01 &           0.48 ± 0.02 &           0.59 ± 0.02 &           0.65 ± 0.02 &           0.68 ± 0.02 &           0.71 ± 0.01 &           0.72 ± 0.01 &           0.74 ± 0.01 &           0.76 ± 0.01 &           0.78 ± 0.02 \\
G-Net  \cite{li2021g}                 &           0.34 ± 0.01 &           0.67 ± 0.03 &           0.83 ± 0.04 &           0.94 ± 0.04 &           1.03 ± 0.05 &           1.10 ± 0.05 &           1.16 ± 0.05 &           1.21 ± 0.06 &           1.25 ± 0.06 &           1.29 ± 0.06 \\
\midrule
EDCT w/ GR ($\lambda = 1$) (\emph{ours}) &           0.29 ± 0.01 &           0.46 ± 0.01 &           0.56 ± 0.01 &           0.62 ± 0.01 &           0.67 ± 0.01 &           0.70 ± 0.01 &           0.72 ± 0.01 &           0.74 ± 0.01 &           0.76 ± 0.01 &           0.78 ± 0.01 \\
CT ($\alpha = 0$) (\emph{ours}) $^\ast$ &  \textbf{0.20 ± 0.01} &  \textbf{0.38 ± 0.01} &  \textbf{0.45 ± 0.01} &           0.50 ± 0.02 &  \textbf{0.52 ± 0.02} &           0.55 ± 0.02 &           0.56 ± 0.02 &           0.58 ± 0.02 &           0.60 ± 0.02 &           0.61 ± 0.02 \\
CT (\emph{ours})                &  \textbf{0.20 ± 0.01} &  \textbf{0.38 ± 0.01} &  \textbf{0.45 ± 0.01} &  \textbf{0.49 ± 0.01} &  \textbf{0.52 ± 0.02} &  \textbf{0.53 ± 0.02} &  \textbf{0.55 ± 0.02} &  \textbf{0.56 ± 0.02} &  \textbf{0.58 ± 0.02} &  \textbf{0.59 ± 0.02} \\
\bottomrule
\multicolumn{11}{l}{Lower $=$ better (best in bold)} \\
\multicolumn{5}{l}{$^\ast$ Identical hyperparameters as proposed \shortname for comparability}
\end{tabular}
\addtolength{\tabcolsep}{1.4pt}
    \end{center}
    \vskip -0.27in
\end{table*}

\subsection{Training of our \longname} \label{sub-sec:CT-training}

In our \shortname, we aim at two simultaneous objectives to address confounding bias: we aim at learning representations that are (a)~predictive of the next outcome and (b)~are non-predictive of the current treatment assignment. This thus naturally yields an adversarial objective. For this purpose, we make use of balanced representations, which we train via a novel \emph{counterfactual domain confusion (CDC) loss}.

\paragraph{Adversarial balanced representations.} 

As in \cite{bica2020estimating}, we build \emph{balanced} representations that allow us to achieve the adversarial objectives (a) and (b). For this, we put two fully-connected networks on top of the representation $\mathbf{\Phi}_t$, corresponding to the respective objectives: (a)~an outcome prediction network $G_Y$ and (b)~a treatment classifier network $G_A$.  Both receive the representation $\mathbf{\Phi}_t$ as input; the outcome prediction network additionally receives the current treatment $\mathbf{A}_t$. We implement both as single hidden layer fully-connected networks with number of units $n_{\text{FC}}$ and ELU activation. For notation, let $\theta_{Y}$ and $\theta_{A}$  denote the trainable parameters in $G_Y$ and $G_A$, respectively. Further, let $\theta_{R}$ denote all trainable parameters in \shortname for generating the representation $\mathbf{\Phi}_t$. 

\paragraph{Factual outcome loss.} For objective~(a), we fit the outcome prediction network $G_Y$, and thus $\mathbf{\Phi}_t$, by minimizing the factual loss of the next outcome. This can be done, \eg, via the mean squared error (MSE). We then yield
\begin{align}
    & \mathcal{L}_{G_Y} (\theta_Y, \theta_R) = \left\Vert \mathbf{Y}_{t+1} - G_Y\big(\mathbf{\Phi}_t(\theta_R), \mathbf{A}_t; \theta_Y \big) \right\Vert^2 .
\end{align}

\paragraph{CDC loss.} For objective~(b), we want to fit the treatment classifier network $G_A$, and thus the representation $\mathbf{\Phi}_t$, in way that it is non-predictive of the current treatment $\mathbf{A}_t$. To achieve this, we develop a novel CDC loss tailored for counterfactual inference. Our idea builds upon the domain confusion loss \cite{tzeng2015simultaneous} for handling adversarial objectives, which was previously used for unsupervised domain adaptation, whereas we adapt it specifically for counterfactual inference. 

Then, we fit $G_A$ so that it can predict the current treatment, \ie, via 
\begin{equation}
\label{eq:loss-ga}
\hspace{-0.3cm}
\mathcal{L}_{G_A} (\theta_A, \theta_R) = - \sum_{j=1}^{d_a} \mathbbm{1}_{[\mathbf{A}_t = a_j]} \log G_A (\mathbf{\Phi}_t(\theta_R); \theta_A) , 
\end{equation}
where $\mathbbm{1}_{[\cdot]}$ is the indicator function. This thus minimizes a classification loss of the current treatment assignment given $\mathbf{\Phi}_t$. However, while $G_A$ can predict the current treatment, the actual representation $\mathbf{\Phi}_t$ should not, and should rather be non-predictive. For this, we propose to minimize the cross-entropy between a uniform distribution over treatment categorical space and predictions of $G_A$ via
\begin{equation}
\label{eq:loss-conf}
\mathcal{L}_{\text{conf}} (\theta_A, \theta_R) = - \sum_{j=1}^{d_a} \frac{1}{d_a} \log G_A (\mathbf{\Phi}_t(\theta_R); \theta_A) ,
\end{equation}
thus achieving domain confusion.

\paragraph{Overall adversarial objective.} 

Using the above, \shortname is trained via
\begin{align}
\hspace{-0.3cm}    
(\hat{\theta}_Y, \hat{\theta}_R) & = \argmin_{\theta_Y, \theta_R} \mathcal{L}_{G_Y} (\theta_Y, \theta_R) + \alpha \mathcal{L}_{\text{conf}} (\hat{\theta}_A, \theta_R) , \label{eq:loss-yr}\\ 
    \hat{\theta}_A & = \argmin_{\theta_A} \alpha \mathcal{L}_{G_A} (\theta_A, \hat{\theta}_R)  , \label{eq:loss-a}
\end{align}
where $\alpha$ is a hyperparameter for domain confusion. Thereby, optimal values of $\hat{\theta}_Y$, $\hat{\theta}_R$ and $\hat{\theta}_A$ achieve an equilibrium between factual outcome prediction and domain confusion. In \shortname, we implement this by performing iterative updates of the parameters (rather than optimizing globally). Details are in Appendix~\ref{app:adv-training}.  

Previous work \cite{bica2020estimating} has addressed the above adversarial objective through gradient reversal \cite{ganin2015unsupervised}. However, this has two shortcomings: (i)~If the parameter $\lambda$ of gradient reversal becomes too large, the representation may be predictive of opposite treatment \cite{atan2018counterfactual}. (ii)~If the treatment classifier network learns too fast, gradients vanish and are not passed to representations, leading to poor fit \cite{tzeng2017adversarial}. Different from that, we propose a novel CDC loss. As we see later, our loss is highly effective: it even improves CRN \cite{bica2020estimating}, when replacing gradient reversal with our loss.

\paragraph{Stabilization.} 

We further stabilize the above adversarial training by employing exponential moving average~(EMA) of model parameters during training \cite{yaz2018unusual}. EMA helps to limit cycles of model parameters around the equilibrium with vanishing amplitude and thus accelerates overall convergence. We apply EMA to all trainable parameters (\ie, $\theta_Y$, $\theta_R$, $\theta_A$). Formally, we update parameters during training via 
\begin{equation}
\theta^{(i)}_{\text{EMA}} = \beta \, \theta^{(i - 1)}_{\text{EMA}} + (1 - \beta) \, \theta^{(i)} ,
\end{equation}
where superscripts $(i)$ refers to the different steps of the optimization algorithm, where $\beta$ is a exponential smoothing parameter, and where we initialize $\theta^{(0)}_{\text{EMA}} = \theta^{(0)}$. We provide pseudocode for an iterative gradient update in \shortname via EMA in Appendix \ref{app:adv-training}.

\paragraph{Attentional dropout.} To reduce the risk of overfitting between time steps, we implement attentional dropout via DropAttention \cite{zehui2019dropattention}. During training, attention scores $\alpha_{ij}$ in Eq.~\eqref{eq:attn-relative-enc} are element-wise randomly set to zero with probability $p$ (\ie, the dropout rate). However, we make a small simplification. We do not perform normalized rescaling \cite{zehui2019dropattention} of attention scores but opt for traditional dropout rescaling \cite{srivastava2014dropout}, as this resulted in more stable training for short-length sequences.

\paragraph{Mini-batch augmentation with masking.}

For training data $\mathcal{D}$, we always have access to the full time-series, that is, including all time-varying covariates $\mathbf{x}_{1}^{(i)}, \dots, \mathbf{x}_{T^{(i)}}^{(i)}$. However, upon deployment, these are no longer observable for $\tau$-step-ahead predictions with $\tau \ge 2$. To reflect this during training, we perform data augmentation at the mini-batch level. For this, we duplicate the training samples: We uniformly sample the length $1 \leq t_s \leq T^{(i)}$ of the masking window, and then create a duplicate data sample where the last $t_s$ time-varying covariates $\mathbf{x}_{t_s}^{(i)}, \dots, \mathbf{x}_{T^{(i)}}^{(i)}$ are masked by setting the corresponding attention logits of $\mathrm{H}^b = \mathrm{X}^b$ in Eq.~\eqref{eq:attention} to $-\infty$.

Mini-batch augmentation with masking allows us to train a single model for both one- and multiple-step-ahead prediction in end-to-end fashion. This distinguishes our \shortname from RMSNs and CRN, which are built on top of encoder-decoder architectures and trained in a multiple-stage procedure. Later, we also experiment with an encoder-decoder version of \shortname (i.e., a single-subnetwork variant) but find that it is inferior performance to our end-to-end model.

\subsection{Theoretical insights}

The following result provides a theoretical justification that our CDC loss indeed leads to balanced representations, and, thus, removes the bias induced by time-varying confounders.\footnote{Importantly, our
loss is different from gradient reversal (GR) in \cite{ganin2015unsupervised, bica2020estimating}. It builds balanced representations by minimizing \emph{reversed KL divergence} between the treatment-conditional distribution of representation and mixture of all treatment-conditional distributions.}
\begin{theorem}\label{thrm:domain_conf_loss_short}
We fix $t \in \mathbb{N}$ and define $P$ as the distribution of $\bar{\mathbf{H}}_t$, $P_j$ as the distribution of $\bar{\mathbf{H}}_t$ given $\mathbf{A}_t = a_j$, and $P^\Phi_j$ as the distribution of $\mathbf{\Phi}_t = \Phi(\bar{\mathbf{H}}_t)$ given $\mathbf{A}_t = a_j$ for all $j \in \{1, \dots, d_a\}$. Here, $\Phi(\cdot) = \Phi(\cdot; \theta_R)$ denotes any network that generates representations. Let $G^j_A$ denote the output of $G_A$ corresponding to treatment $a_j$. Then, there exists an optimal pair $(\Phi^\ast, G^\ast_A)$ such that
\begin{align}\
 \Phi^\ast &= \argmax_{\Phi} \sum_{j=1}^{d_a} \mathbb{E}_{\bar{\mathbf{H}}_t \sim P}\left[ \log{G^\ast}^j_A(\Phi(\bar{\mathbf{H}}_t) \right] \label{eq:phi-star}\\
 G^\ast_A &= \argmax_{G_A} \sum_{j=1}^{d_a} \mathbb{E}_{\bar{\mathbf{H}}_t \sim P_j}\left[\log {G}^j_A(\Phi^\ast(\bar{\mathbf{H}}_t) \right] \mathbb{P}(\mathbf{A}_t = a_j) \label{eq:ga-star}\\
& \text{subject to }  \sum_{i=1}^{d_a} {G}^i_A(\Phi^\ast(\bar{\mathbf{H}}_t)) = 1.
\end{align}
Furthermore, $\Phi^\ast$ satisfies Eq.~\eqref{eq:phi-star} if and only if it induces balanced representations across treatments, i.e., $P^{\Phi^\ast}_1 = \ldots = P^{\Phi^\ast}_{d_a}$.
\end{theorem}
\begin{proof}
See Appendix~\ref{app:proof}.
\end{proof}

Further, it can be easily shown that objectives \eqref{eq:loss-ga} and \eqref{eq:loss-conf} are exactly finite sample versions of \eqref{eq:ga-star} and \eqref{eq:phi-star} from Theorem~\ref{thrm:domain_conf_loss_short}, respectively.

\subsection{Implementation}

\paragraph{Training.} We implemented \shortname in PyTorch Lightning. We trained \shortname using Adam \cite{kingma2014adam} with learning rate $\eta$ and number of epochs $n_e$. The dropout rate $p$ was kept the same for both feed-forward layers and DropAttention (we call it sequential dropout rate). We employed the teacher forcing technique \cite{williams1989learning}. During evaluation of multiple-step-ahead prediction, we switch off teacher forcing and autoregressively feed model predictions. For the parameters $\alpha$ and $\beta$ of adversarial training, we choose values $\beta = 0.99$ and $\alpha = 0.01$ as in the original works \cite{tzeng2015simultaneous,yaz2018unusual}, which also performed well in our experiments. We additionally perform an exponential rise of $\alpha$ during training. 

\paragraph{Hyperparameter tuning.}  $p$, $\eta$, and all other hyperparameters (number of blocks $B$, minibatch size, number of attention heads $n_h$, size of hidden units $d_h$, size of balanced representation $d_r$, size of hidden units in fully-connected networks $n_{\text{FC}}$) are subject to hyperparameter tuning. Details are in Appendix~\ref{app:hparams}.

\begin{table}[tbp]
    \vskip -0.075in
    \caption{Results for experiments with real-world medical data (MIMIC-III). Shown: RMSE as mean $\pm$ standard deviation over five runs.}
    \label{tab:mimic-real-sim-all}
    \begin{center}
        \tiny
        \addtolength{\tabcolsep}{-1.5pt}  
\begin{tabular}{l|c|cccc}
\toprule
{} &            $\tau = 1$ &            $\tau = 2$ &            $\tau = 3$ &            $\tau = 4$ &             $\tau = 5$ \\
\midrule
MSMs       &           6.37 ± 0.26 &           9.06 ± 0.41 &          11.89 ± 1.28 &          13.12 ± 1.25 &           14.44 ± 1.12 \\
RMSNs      &           5.20 ± 0.15 &           9.79 ± 0.31 &          10.52 ± 0.39 &          11.09 ± 0.49 &           11.64 ± 0.62 \\
CRN       &  4.84 ± 0.08 &           9.15 ± 0.16 &           9.81 ± 0.17 &          10.15 ± 0.19 &           10.40 ± 0.21 \\
G-Net     &           5.13 ± 0.05 &          11.88 ± 0.20 &          12.91 ± 0.26 &          13.57 ± 0.30 &           14.08 ± 0.31 \\
\midrule
CT (\emph{ours}) &           \textbf{4.59 ± 0.09} &  \textbf{8.99 ± 0.21} &  \textbf{9.59 ± 0.22} &  \textbf{9.91 ± 0.26} &  \textbf{10.14 ± 0.29} \\
\bottomrule
\multicolumn{6}{l}{Lower $=$ better (best in bold) }
\end{tabular}
\addtolength{\tabcolsep}{1.5pt}  
    \end{center}
    \vskip -0.29in
\end{table}

\begin{table}[tbp]
    \caption{Ablation study for proposed \shortname (with CDC loss, $\alpha = 0.01$, $\beta = 0.99$). Reported: normalized RMSE of \shortname with relative changes.}
    \label{tab:ablation-study}
    \begin{center}
        \scriptsize
        \addtolength{\tabcolsep}{-1.6pt}  
        \begin{tabular}{l|l|rr|rr}
\toprule
\multicolumn{2}{c|}{} & \multicolumn{2}{c|}{$\tau$ = 1} & \multicolumn{2}{c}{$\tau$ = 6} \\
\cmidrule(lr){3-4}\cmidrule(lr){5-6}
\multicolumn{2}{c|}{}  &            $\gamma$ = 1 &             $\gamma$ = 4 &            $\gamma$ = 1 &             $\gamma$ = 4 \\
\midrule

\multicolumn{2}{l|}{CT (proposed)} &                     0.80 &                     1.32 &                    0.63 &                     0.93 \\
\multirow{4}{*}{\textsf{a}} & w/ non-trainable PE$^*$                                  &   \textcolor{gray}{$\pm$0.00} &  \textcolor{ForestGreen}{$-$0.02} &  \textcolor{BrickRed}{$+$0.01} &  \textcolor{ForestGreen}{$-$0.03} \\
& w/ absolute PE$^*$                                       &  \textcolor{BrickRed}{$+$0.04} &   \textcolor{BrickRed}{$+$0.16} &  \textcolor{BrickRed}{$+$0.15} &    \textcolor{BrickRed}{$+$1.00} \\
& w/o attentional dropout$^*$                            &  \textcolor{gray}{$\pm$0.00} &   \textcolor{BrickRed}{$+$0.07} &   \textcolor{BrickRed}{$+$0.00} &   \textcolor{BrickRed}{$+$0.09} \\
\vspace{0.9mm}
& w/o cross-attention$^*$                                      &   \textcolor{BrickRed}{$+$0.03} &   \textcolor{BrickRed}{$+$0.16} &   \textcolor{BrickRed}{$+$0.06} &    \textcolor{BrickRed}{$+$0.10} \\

\multirow{3}{*}{\textsf{b}} & w/o EMA ($\beta = 0$)$^*$                                       &  \textcolor{BrickRed}{$+$0.03} &   \textcolor{BrickRed}{$+$0.38} &  \textcolor{BrickRed}{$+$0.03} &   \textcolor{BrickRed}{$+$0.33} \\
& w/o balancing ($\alpha$ = 0; $\beta$ = 0.99)$^*$                  &  \textcolor{ForestGreen}{$-$0.01} &  \textcolor{ForestGreen}{$-$0.02} &   \textcolor{gray}{$\pm$0.00} &   \textcolor{BrickRed}{$+$0.07} \\
\vspace{0.9mm}
& w/ GR ($\lambda$ = 1) &  \textcolor{BrickRed}{$+$0.02} &   \textcolor{BrickRed}{$+$0.17} &  \textcolor{BrickRed}{$+$0.08} &   \textcolor{BrickRed}{$+$0.33} \\

\multirow{2}{*}{\textsf{c}} & EDCT w/ GR ($\lambda$ = 1) &  \textcolor{BrickRed}{$+$0.16} &   \textcolor{BrickRed}{$+$0.08} &  \textcolor{BrickRed}{$+$0.05} &   \textcolor{BrickRed}{$+$0.23} \\
& EDCT w/ DC ($\alpha$ = 0.01; $\beta$ = 0.99) &  \textcolor{ForestGreen}{$-$0.03} &    \textcolor{BrickRed}{$+$0.10} &  \textcolor{ForestGreen}{$-$0.03} &   \textcolor{BrickRed}{$+$0.23} \\
\bottomrule
\multicolumn{5}{l}{Lower = better;} \\
\multicolumn{5}{l}{Improvement over \shortname in \textcolor{ForestGreen}{green}, worse performance in \textcolor{BrickRed}{red}} \\
\multicolumn{5}{l}{$^\ast$ Identical hyperparameters as proposed \shortname for comparability}
\end{tabular}

        \addtolength{\tabcolsep}{1.6pt}  
    \end{center}
    \vskip -0.29in
\end{table}

\begin{table*}[tp]
    \vskip -0.075in
    \caption{CRN with different training procedures. Results for fully-synthetic data based on tumor growth simulator (here: $\gamma = 4$).}
    \label{tab:ablation-study-crn}
    \begin{center}
        \scriptsize
        \begin{tabular}{l|c|ccccc}
\toprule
{} &            $\tau$ = 1 &            $\tau$ = 2 &            $\tau$ = 3 &            $\tau$ = 4 &            $\tau$ = 5 &            $\tau$ = 6 \\
\midrule
CRN + original GR ($\lambda = 1$) as in \cite{bica2020estimating}        &  \textbf{1.30 ± 0.14} &  \textbf{1.12 ± 0.25} &           1.23 ± 0.32 &           1.23 ± 0.34 &           1.17 ± 0.34 &           1.10 ± 0.32 \\
CRN + our counterfactual DC loss ($\alpha = 0.01$; $\beta = 0.99$) &           1.33 ± 0.21 &           1.18 ± 0.31 &  \textbf{1.19 ± 0.36} &  \textbf{1.12 ± 0.35} &  \textbf{1.03 ± 0.33} &  \textbf{0.93 ± 0.31} \\
\bottomrule
\multicolumn{7}{l}{Lower $=$ better (best in bold) }
\end{tabular}

    \end{center}
    \vskip -0.25in
\end{table*}

\section{Experiments}

To demonstrate the effectiveness of our \shortname, we make use of synthetic datasets. Thereby, we follow common practice in benchmarking for counterfactual inference \cite{lim2018forecasting,bica2020estimating,li2021g}. For real datasets, the true counterfactual outcomes are typically unknown. By using \mbox{(semi-)}synthetic datasets, we can compute the true counterfactuals and thus validate our \shortname.    

\paragraph{Baselines.} The chosen baselines are identical to those in previous, state-of-the-art literature for estimating counterfactual outcomes over time \cite{lim2018forecasting,bica2020estimating,li2021g}. These are: \textbf{MSMs}~\cite{robins2000marginal,hernan2001marginal}, \textbf{RMSNs}~\cite{lim2018forecasting}, \textbf{CRN}~\cite{bica2020estimating}, and \textbf{G-Net} \cite{li2021g}. Details are in Appendix~\ref{app:baselines}. For comparability, we use the same hyperparameter tuning for the baselines as for \shortname (see Appendix~\ref{app:hparams}).

\subsection{Experiments with fully-synthetic data} \label{sec:tg-sim}

\paragraph{Data.} We build upon the pharmacokinetic-pharmacodynamic model of tumor growth \cite{geng2017prediction}. It provides a state-of-the-art biomedical model to simulate the effects of lung cancer treatments over time. The same model was previously used for evaluating RMSNs \cite{lim2018forecasting} and CRN \cite{bica2020estimating}. For $\tau$-step-ahead prediction, we distinguish two settings: (i)~``single sliding treatment'' where trajectories involve only a single treatment as in \citep{bica2020estimating}; and (ii)~``random trajectories'' where one or more treatments are assigned. We simulate patient trajectories for different amounts of confounding $\gamma$. Further details are in Appendix~\ref{app:syn}. Here, and in all following experiments, we apply hyperparameter tuning (see Appendix~\ref{app:hparams}).

\paragraph{Results.} Fig.~\ref{fig:results-tg-sim} shows the results. We see a notable performance gain for our \shortname over the state-of-the-art baselines, especially pronounced for larger confounding $\gamma$ and larger $\tau$. Overall, \shortname is superior by a large margin. 

Fig.~\ref{fig:results-tg-sim} also shows a \shortname variant in which we removed the CDC loss by setting $\alpha$ to zero, called \shortname($\alpha = 0$). For comparability, we keep the hyperparameters as in the original \shortname. The results demonstrate the effectiveness of the proposed CDC loss, especially for multi-step-ahead prediction. \shortname also provides a significant runtime speedup in comparison to other neural network methods, mainly due to faster processing of sequential data with self- and cross-attentions, and single-stage end-to-end training (see exact runtime and model size comparison in Appendix~\ref{app:runtime}). We plotted t-SNE embeddings of the balanced representations (Appendix~\ref{app:t-sne}) to exemplify how balancing works.

\subsection{Experiments with semi-synthetic data}

\paragraph{Data.} We create a semi-synthetic dataset based on real-world medical data from intensive care units. This allows us to validate our \shortname with high-dimensional, long-range patient trajectories. For this, we use the MIMIC-III dataset \cite{johnson2016mimic}. Building upon the ideas of \cite{schulam2017reliable}, we then generate patient trajectories with outcomes under endogeneous and exogeneous dependencies while considering treatment effects. Thereby, we can again control for the amount of confounding. Details are in Appendix~\ref{app:ss-sim}. Importantly, we again have access to the ground-truth counterfactuals for evaluation.

\paragraph{Results.} Table~\ref{tab:ss-sim-all} shows the results. Again, \shortname has a consistent and large improvement across all projection horizons $\tau$ (average improvement over baselines: 38.5\%). By comparing our \shortname against \shortname($\alpha = 0$), we see clear performance gains, demonstrating the benefit of our CDC loss. Additionally, we separately fitted an encoder-decoder architecture, namely \emph{Encoder-Decoder} \longname (EDCT). This approach leverages a single-subnetwork architecture, where all three sequences are fed into a single subnetwork (as opposed to three separate networks as in our \shortname). Further, the EDCT leverages the existing GR loss from \cite{bica2020estimating} and the similar encoder-decoder two-stage training. Details on this EDCT model are in Appendix~\ref{app:EDCT}. Here, we find that, for superior performance, it is crucial to combine both three-subnetwork architecture \underline{and} our CDC loss. 

Semi-synthetic data is also used for a case study, where we study the importance of each subnetwork. See Appendix~\ref{app:subnetwork-importance}.

\subsection{Experiments with real-world data}

\paragraph{Data.} We now demonstrate the applicability of our \shortname to real-world data and, for this, use intensive care unit stays in MIMIC-III \cite{johnson2016mimic}. We use the same 25 vital signs and 3 static features. We use (diastolic) blood pressure as an outcome and consider two treatments: vasopressors and mechanical ventilation, similar to \cite{kuzmanovic2021deconfounding, hatt2021sequential}. Prediction of blood pressure is crucial for critical care, \eg, to avoid tissue hypoperfusion \cite{vincent2018mean}. The application of vasopressors is highly confounded by previous and current levels of blood pressure, as they aim to raise low blood pressure. So far, an optimal administration of vasopressors is not fully understood \cite{subramanian2008liberal}, and, hence, it is important for medical practitioners to have individualized counterfactual predictions. Experiment details are in Appendix~\ref{app:real-world-data}. 

\paragraph{Results.} Because we no longer have access to the true counterfactuals, we now report the performance of predicting factual outcomes; see Table~\ref{tab:mimic-real-sim-all}. All state-of-the-art baselines are outperformed by our \shortname. This demonstrates the superiority of our proposed model.

\subsection{Ablation study} \label{sec:ablation}

We performed an extensive ablation study (Table~\ref{tab:ablation-study}) using full-synthetic data (setting: random trajectories) to confirm the effectiveness of the different components inside the subnetworks, the CDC loss, and the subnetwork architecture. We grouped these into categories. \textsf{a} varies different components within the subnetworks. Here, we replace trainable relative positional encoding (PE) with non-trainable relative PE, generated as described in Appendix~\ref{app:abs-pe}. Further, we replace our PE with a trainable absolute PE as in the original transformer \cite{vaswani2017attention}. Finally, we remove attentional dropout as well as cross-attention layers for all subnetworks. \textsf{b} varies the loss. Here, we remove EMA of model weights; switch off adversarial balancing, but not EMA; and replace our CDC loss with gradient reversal (GR) as in \cite{bica2020estimating}. \textsf{c} evaluates a single-subnetwork version of \shortname. We refer to this as EDCT (see Appendix~\ref{app:EDCT} for details). It thus has an encoder-decoder architecture which we train with either our CDC loss or GR.

Overall, we see that the combination of both our novel architecture based three-subnetworks \underline{and} our novel DC loss is crucial. This observation is particularly pronounced for a long prediction horizon ($\tau = 6$), where our proposed \shortname achieves the best performance. Notably, the main insight here is: simply switching the backbone from LSTM to transformer and using gradient reversal as in \cite{bica2020estimating} gives unstable results (see ``EDCT w/ GR ($\lambda$ = 1)``). Furthermore, our three-subnetworks \shortname with GR loss performs even worse (see ablation ``w/ GR ($\lambda$ = 1)``). 

To further demonstrate the effectiveness of our novel CDC loss, we perform an additional test based on the fully-synthetic dataset (Table~\ref{tab:ablation-study-crn}). We use (i)~a CRN with GR as in \cite{bica2020estimating}. We compare it with (ii)~a CRN trained with our proposed CDC loss (implementation details in Appendix~\ref{app:baselines}). Evidently, our loss also helps the CRN to achieve a better RMSE.


\section{Conclusion}

For personalized medicine, estimates of the counterfactual outcomes for patient trajectories are needed. Here, we proposed a novel, state-of-the-art method: the \longname which is designed to capture complex, long-range patient trajectories. It combines a custom subnetwork architecture to process the input together with a new counterfactual domain confusion loss for end-to-end training. Across extensive experiments, our \longname achieves state-of-the-art performance.



\FloatBarrier
\bibliography{literature}
\bibliographystyle{icml2022}

\newpage
\appendix
\onecolumn

\section{Assumptions for Causal Identification} 
\label{app:assumptions}

We build upon the potential outcomes framework \cite{neyman1923application,rubin1978bayesian} and its extension to time-varying treatments and outcomes \cite{robins2009estimation}. The potential outcomes framework has been widely used in earlier works with a similar objective as ours \cite{robins2009estimation, lim2018forecasting, bica2020estimating}. To this end, three standard assumptions for data generating mechanism are needed to identify a counterfactual outcome distribution over time, or, specifically, average $\tau$-step-ahead potential outcome conditioned on history from Eq.~\eqref{eq:estimand}: 
\begin{assumption}
    \label{ass:consistency}
    \textbf{(Consistency)}. If $\bar{\mathbf{A}}_{t} = \bar{\mathbf{a}}_{t}$ is a given sequence of treatments for some patient, then $\mathbf{Y}_{t+1}[\bar{\mathbf{a}}_{t}] = \mathbf{Y}_{t+1}$. This means that the potential outcome under treatment sequence $\bar{\mathbf{a}}_{t}$ coincides for the patient with the observed (factual) outcome, conditional on $\bar{\mathbf{A}}_{t} = \bar{\mathbf{a}}_{t}$.
\end{assumption}
\begin{assumption}
    \label{ass:so}
    \textbf{(Sequential Overlap)}. There is always a non-zero probability of receiving/not receiving any treatment for all the history space over time: $0 < \mathbb{P}(\mathbf{A}_{t} = \mathbf{a}_{t}  \;\mid\;  \bar{\mathbf{H}}_{t} = \bar{\mathbf{h}}_{t}) < 1$, if $\mathbb{P}(\bar{\mathbf{H}}_{t} = \bar{\mathbf{h}}_{t}) > 0$, where $\bar{\mathbf{h}}_{t}$ is some realization of a patient history.
\end{assumption}
\begin{assumption}\label{ass:seq_ignor}
    \label{ass:si}
    \textbf{(Sequential Ignorability)} or No Unobserved Confounding. The current treatment is independent of the potential outcome, conditioning on the observed history: $\mathbf{A}_{t} \ind \mathbf{Y}_{t+1}[\mathbf{a}_t] \,\mid\, \bar{\mathbf{H}}_{t}, \quad \forall \mathbf{a}_t$. This implies that there are no unobserved confounders that affect both treatment and outcome. 
\end{assumption}

The data generating mechanism for $\mathcal{D}$ is shown in Figure~\ref{fig:causal_diagram}. 

\begin{figure}[ht]
    \centering
    \includegraphics[width=0.95\textwidth]{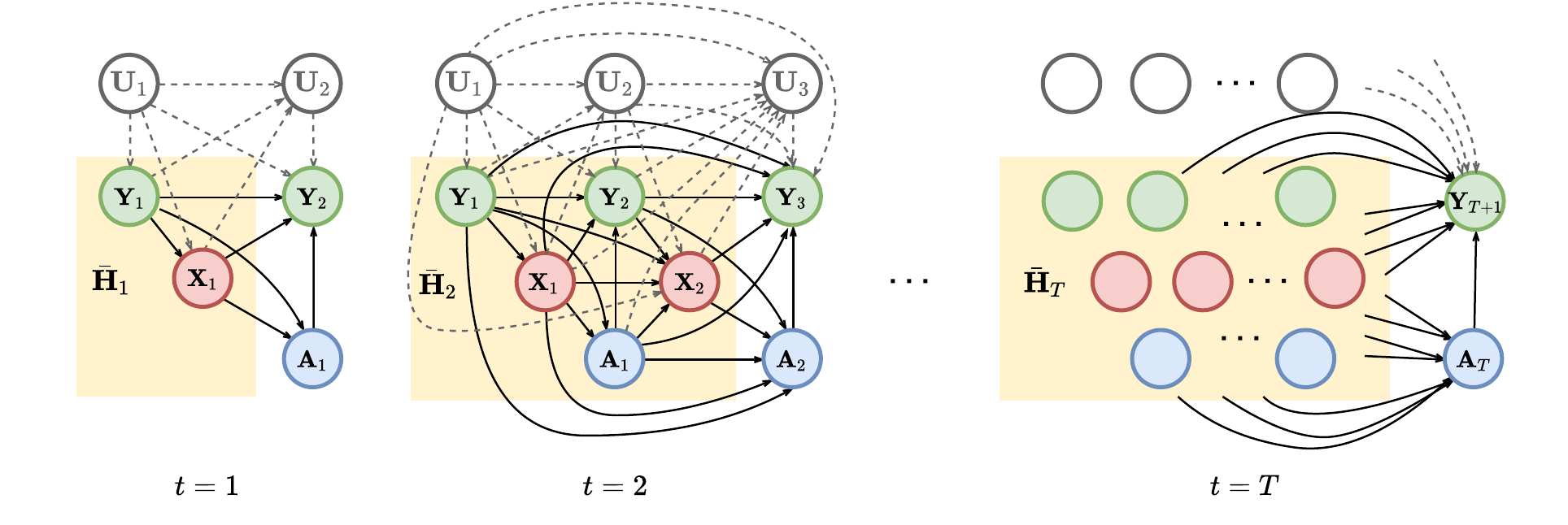}
    \caption{Causal diagram for data generating mechanism, depicted for different time steps $t$. $\mathbf{U}_t$ is unobserved exogenous noise, which only affects time-varying covariates and outcomes, but not treatments. All time-varying confounders up to time $t$ are included in the observed history $\bar{\mathbf{H}}_t$. Static covariates are ignored for the simplicity.}
    \label{fig:causal_diagram}
\end{figure}

\begin{corollary} (G-computation \cite{robins1986new}). Assumptions \ref{ass:consistency}--\ref{ass:si} provide sufficient identifiability conditions for Eq.~\eqref{eq:estimand}, e.g., with the G-computation 

\begin{equation}
\label{eq:gformula}
\begin{split}
       & \mathbb{E} \left(\mathbf{Y}_{t + \tau}[\bar{\mathbf{a}}_{t: t+\tau-1}] \mid \bar{\mathbf{H}}_t\right) = \int_{\mathbb{R}^{d_x} \times \dots \times \mathbb{R}^{d_x}}  \mathbb{E} \left(\mathbf{Y}_{t+\tau} \;\middle|\; \bar{\mathbf{H}}_t, \bar{\mathbf{x}}_{t+1:t+\tau-1}, \bar{\mathbf{y}}_{t+1:t+\tau-1}, \bar{\mathbf{a}}_{t: t+\tau-1} \right) \times \\ 
       & \prod_{j=t+1}^{t+\tau-1}  \mathbb{P} \left(\mathbf{x}_{j} \mathbf{y}_{j} \mid  \bar{\mathbf{H}}_t, \bar{\mathbf{x}}_{t+1:j-1}, \bar{\mathbf{y}}_{t+1:j-1}, \bar{\mathbf{a}}_{t:j-1} \right) \; \dd\bar{\mathbf{x}}_{t+1:t+\tau-1} \dd\bar{\mathbf{y}}_{t+1:t+\tau-1}
\end{split}
\end{equation}
\end{corollary}

Empirical G-computation is used by G-Net \cite{li2021g}, but requires the estimation of conditional distributions of time-varying covariates. This could be particularly challenging, given a finite dataset size and high dimensionality of covariates. Thus, we refrain from explicit usage of G-computation.  

\clearpage
\section{Methods for Estimating Counterfactual Outcomes over Time}  \label{app:methods-table}

In Table~\ref{tab:methods-comparison}, we provide an overview of the machine-learning-based methods for estimating counterfactual outcomes over time. For our experiments, we selected \textbf{MSMs} \cite{robins2000marginal,hernan2001marginal}, as the simplest linear baseline, and three state-of-the-art methods: \textbf{RMSNs} \cite{lim2018forecasting}, \textbf{CRN} \cite{bica2020estimating}, and \textbf{G-Net} \cite{li2021g}). Importantly, our choice of baselines is thus analogous to the those in the state-of-the-art literature \cite{lim2018forecasting,bica2020estimating,li2021g}. Below, we provide details why certain works are not of fit for our setting and are thus not applicable as baselines. Here, we again emphasize that this selection is again consistent with the literature \cite{lim2018forecasting,bica2020estimating,li2021g}.

One stream of the literature focuses on non- or semi-parametric methods \cite{Xu16,schulam2017reliable,soleimani2017treatment}. These are, for example, based on Gaussian processes (GPs). However, the aforementioned methods have three limitations: (1)~They are not designed to handle static covariates. As such, risk factors (e.g., age, gender) that are standard in any electronic health record must be excluded. This omits a substantial heterogeneity in any patient cohort, and is thus impractical. (2)~These methods cannot handle multiple outcomes. (3)~Due to the non-parametric nature of their estimation, these methods typically cannot scale to large-scale datasets. Contrary to that, several methods have been built that overcome these limitations, namely \textbf{RMSNs} \cite{lim2018forecasting}, \textbf{CRN} \cite{bica2020estimating}, and \textbf{G-Net} \cite{li2021g}, which we included as baselines.

We excluded two additional methods as these do not match our setting: 
\begin{itemize}
\item SyncTwin \cite{qian2021synctwin} is a semi-parametric method using synthetic control, but is limited to a single-time binary treatment and, therefore, not applicable to our setting.
\item DCRN \cite{berrevoets2021disentangled} may appear relevant at first glance; however, it only works with sequences of binary treatments. More importantly, it requires a stronger version of the sequential ignorability assumption: \emph{Sequential Ignorability conditional on current covariates. The current treatment is independent from the potential outcome, conditional on current time-varying covariates: $\mathbf{A}_{t} \ind \mathbf{Y}_{t+1}[\mathbf{a}_t] \,\mid\, \mathbf{X}_{t}, \quad \forall \mathbf{a}_t$.} Therefore, this setting and ours are  different, as in our setting, past time-varying covariates may also serve as confounders.
\end{itemize}

\noindent
We further make an important remark. The problem of estimating counterfactual outcomes over time differs also from reinforcement learning: different from reinforcement learning, we assume a non-Markovian data generation mechanism. This impedes the applicability of such approaches as the size of the state space (history) grows typically with time (see Appendix~\ref{app:assumptions}, Fig.~\ref{fig:causal_diagram}).

\begin{table*}[h!]
    \caption{Overview of methods for estimating counterfactual outcomes over time.}
    \label{tab:methods-comparison}
    \begin{center}
        \scriptsize
                \begin{tabular}{l lllll}
                    \toprule
                    Method & Setting & Model type (backbone) & Time & Treatments & Framework \\
                    \midrule
                    HITR \cite{Xu16} & DGM (\xmark) & NP (GP) & Disc \& Cont & Seq, Cat & G-computation \\
                    CGP \cite{schulam2017reliable} & C, SO, SI, CSI (\xmark) & NP (GP) & Cont & Seq, Cat & G-computation \\
                    MOGP \cite{soleimani2017treatment} & DGM (\xmark) & SP (GP) & Disc \& Cont & Seq, Cont & G-computation \\
                    SyncTwin \cite{qian2021synctwin} & DGM (\xmark) & SP (GRU-D, LSTM) & Disc & Single-time, Bin & Synthetic control \\
                    DCRN \cite{berrevoets2021disentangled} & C, SO, Cov (\xmark) & P (3 LSTMs) & Disc &  Seq, Bin & Disentangled representation \\   
                    \midrule
                    * MSMs \cite{robins2000marginal} & C, SO, SI (\cmark) & P (Logistic \& linear regressions) & Disc & Seq, Cat & IPTW weighted loss \\
                    * RMSNs \cite{lim2018forecasting} & C, SO, SI (\cmark) & P (LSTM) & Disc & Seq, Cat & IPTW weighted loss \\
                    * CRN \cite{bica2020estimating}  & C, SO, SI (\cmark) & P (LSTM) & Disc & Seq, Cat & BR (gradient reversal) \\
                    * G-Net \cite{li2021g} & C, SO, SI (\cmark) & P (LSTM) & Disc & Seq, Cat & G-computation \\
                    \midrule
                    * \longname (this paper)              & C, SO, SI & P (3 transformers) & Disc & Seq, Cat & BR (CDC) \\
                    \bottomrule
                    \multicolumn{6}{l}{$^\ast$ $=$ Methods with the same assumptions as ours (and thus included in our baselines)} \\
                    \multicolumn{6}{p{14cm}}{\emph{Legend:}}\\
                    \multicolumn{6}{p{14cm}}{\quad $\bullet$ Setting: consistency (C), sequential overlap (SO), sequential ignorability (SI), sequential ignorability but conditional on covariates (Cov), continuous sequential ignorability (CSI), assumed data generating model (DGM)}\\
                    \multicolumn{6}{p{14cm}}{\quad $\bullet$ Model: parametric (P), semi-parametric (SP), and non-parametric (NP)}\\
                    \multicolumn{6}{p{14cm}}{\quad $\bullet$ Time: discrete (Disc) or continuous (Cont) time steps}\\
                    \multicolumn{6}{p{14cm}}{\quad $\bullet$ Treatments: sequential (Seq), binary (Bin), categorical (Cat), continuous (Cont).}\\
                    \multicolumn{6}{p{14cm}}{\quad $\bullet$ Framework: inverse probability of treatment weights (IPTW), balanced representations (BR)}
                \end{tabular}
    \end{center}
\end{table*}

\clearpage
\section{Details for Transformer Block} \label{app:CT-block}

Here, we provide the detailed formalization of the multi-input transformer block. Recall that our multi-input transformer block builds on top of three intertwined transformer subnetworks (see Fig.~\ref{fig:multi-input-transformer}). First, we incorporate three separate self-attentions:
\begin{align}
        & \tilde{\mathrm{A}}^{b - 1} = \operatorname{LN}\Big(\operatorname{MHA}\big(Q(\mathrm{A}^{b - 1}), K(\mathrm{A}^{b - 1}), V(\mathrm{A}^{b - 1})\big) + \mathrm{A}^{b - 1}\Big) , \\
        & \tilde{\mathrm{X}}^{b - 1} = \operatorname{LN}\Big(\operatorname{MHA}\big(Q(\mathrm{X}^{b - 1}), K(\mathrm{X}^{b - 1}), V(\mathrm{X}^{b - 1})\big) + \mathrm{X}^{b - 1}\Big) , \\
        & \tilde{\mathrm{Y}}^{b - 1} = \operatorname{LN}\Big(\operatorname{MHA}\big(Q(\mathrm{Y}^{b - 1}), K(\mathrm{Y}^{b - 1}), V(\mathrm{Y}^{b - 1})\big) + \mathrm{Y}^{b - 1}\Big) .
\end{align}
Further, we incorporate cross-attentions:
\begin{align}
        & \tilde{\mathrm{A}}^{b - 1}_{\mathrm{X}} = 
        \operatorname{LN}\Big(\operatorname{MHA}\big(Q(\tilde{\mathrm{A}}^{b - 1}), K(\mathrm{X}^{b - 1}), V(\mathrm{X}^{b - 1})\big) + \tilde{\mathrm{A}}^{b - 1}\Big) , \\ 
        & \tilde{\mathrm{A}}^{b - 1}_{\mathrm{Y}} = 
        \operatorname{LN}\Big(\operatorname{MHA}\big(Q(\tilde{\mathrm{A}}^{b - 1}), K(\mathrm{Y}^{b - 1}), V(\mathrm{Y}^{b - 1})\big) + \tilde{\mathrm{A}}^{b - 1}\Big) , \\[0.3cm]
        & \tilde{\mathrm{X}}^{b - 1}_{\mathrm{A}} = 
        \operatorname{LN}\Big(\operatorname{MHA}\big(Q(\tilde{\mathrm{X}}^{b - 1}), K(\mathrm{A}^{b - 1}), V(\mathrm{A}^{b - 1})\big) + \tilde{\mathrm{X}}^{b - 1}\Big) , \\ 
        & \tilde{\mathrm{X}}^{b - 1}_{\mathrm{Y}} = 
        \operatorname{LN}\Big(\operatorname{MHA}\big(Q(\tilde{\mathrm{X}}^{b - 1}), K(\mathrm{Y}^{b - 1}), V(\mathrm{Y}^{b - 1})\big) + \tilde{\mathrm{X}}^{b - 1}\Big) , \\[0.3cm]
        & \tilde{\mathrm{Y}}^{b - 1}_{\mathrm{X}} = 
        \operatorname{LN}\Big(\operatorname{MHA}\big(Q(\tilde{\mathrm{Y}}^{b - 1}), K(\mathrm{X}^{b - 1}), V(\mathrm{X}^{b - 1})\big) + \tilde{\mathrm{Y}}^{b - 1}\Big) , \\ 
        & \tilde{\mathrm{Y}}^{b - 1}_{\mathrm{A}} = 
        \operatorname{LN}\Big(\operatorname{MHA}\big(Q(\tilde{\mathrm{Y}}^{b - 1}), K(\mathrm{A}^{b - 1}), V(\mathrm{A}^{b - 1})\big) + \tilde{\mathrm{Y}}^{b - 1}\Big) .
\end{align}
Notably, the tensors of treatment representations $\mathrm{A}^{b}$ and $\tilde{\mathrm{A}}^{b}$ are left-shifted with respect to covariates and outcomes representation tensors. Next, we pool the intermediate outputs using linearly transformed static features:
\begin{align}
        & \tilde{\tilde{\mathrm{A}}}^{b - 1} = \tilde{\mathrm{A}}^{b - 1}_{\mathrm{X}} + \tilde{\mathrm{A}}^{b - 1}_{\mathrm{Y}} + \mathbf{1} \tilde{\mathbf{V}}^\top , \\
        & \tilde{\tilde{\mathrm{X}}}^{b - 1} = \tilde{\mathrm{X}}^{b - 1}_{\mathrm{A}} + \tilde{\mathrm{X}}^{b - 1}_{\mathrm{Y}} + \mathbf{1} \tilde{\mathbf{V}}^\top , \\
        & \tilde{\tilde{\mathrm{Y}}}^{b - 1} = \tilde{\mathrm{Y}}^{b - 1}_{\mathrm{X}} + \tilde{\mathrm{Y}}^{b - 1}_{\mathrm{A}} + \mathbf{1} \tilde{\mathbf{V}}^\top .
\end{align}
Finally, the hidden states are processed in parallel by feed-forward layers:
\begin{align}
        & \mathrm{A}^{b} = \operatorname{LN}\big(\operatorname{FF}(\tilde{\tilde{\mathrm{A}}}^{b - 1}) + \tilde{\tilde{\mathrm{A}}}^{b - 1}\big) \\
        & \mathrm{X}^{b} = \operatorname{LN}\big(\operatorname{FF}(\tilde{\tilde{\mathrm{X}}}^{b - 1}) + \tilde{\tilde{\mathrm{X}}}^{b - 1}\big) , \\
        & \mathrm{Y}^{b} = \operatorname{LN}\big(\operatorname{FF}(\tilde{\tilde{\mathrm{Y}}}^{b - 1}) + \tilde{\tilde{\mathrm{Y}}}^{b - 1}\big) .
\end{align}

\clearpage
\section{Absolute Positional Encoding} 
\label{app:abs-pe}

For completeness, we briefly summarize \emph{absolute} positional encoding \cite{vaswani2017attention} in the following. We thereby hope that readers can better understand the differences in the use of \emph{relative} positional encoding as used in our \longname.

In \cite{vaswani2017attention}, absolute positional encoding $\operatorname{PE}(t) \in \mathbb{R}^{d_h}$ was introduced to uniquely encode each time step $t \in \{1, \dots, T\}$. Absolute positional encodings were added to the initial hidden states right before the first transformer block via
\begin{equation}
    \hat{\mathbf{h}}_t^0 = \mathbf{h}_t^0 + \operatorname{PE}(t) .
\end{equation}
In addition, the authors used fixed (non-trainable) weights, produced by sine and cosine functions with differing frequencies; \ie, 
\begin{align}
    \label{eq:pe-sin}
    \big(\operatorname{PE}(t)\big)_{2j} & = \sin{\frac{t}{10000^{2j/d_h}}} , \\
    \label{eq:pe-cos}
    \big(\operatorname{PE}(t)\big)_{2j + 1} & = \cos{\frac{t}{10000^{2j/d_h}}} .
\end{align}
This encoding scheme ensures continuity between neighboring time steps and that time-delta shifts are equivalent to linear transformations. Alternatively, one could use trainable absolute positional encodings, which would require learning $T \times d_h$ parameters, where $T$ is a maximum sequence length from the training subset. For our \shortname, this limits the ability to generalize to unseen sequence lengths. Hence, this is the main reason why we opted for clipped relative positional encodings for our \shortname instead.

Notably, we used the same fixed encoding scheme from Eq.~\eqref{eq:pe-sin}--\eqref{eq:pe-cos} to produce non-trainable relative PE. 

\clearpage
\section{Details on Adversarial Training} \label{app:adv-training}

In our \shortname, we implement the adversarial training for Eq.~\eqref{eq:loss-yr} and Eq.~\eqref{eq:loss-a} by performing iterative gradient descent updates (rather than optimizing globally). Algorithm~\ref{alg:grad-desc-ema} provides the pseudocode. Recall that we further make use of exponential moving average (EMA) for stabilization. 

\begin{algorithm}[htb]
  \caption{ Adversarial training in \shortname via iterative gradient descent}
  \label{alg:grad-desc-ema}
    \begin{algorithmic}
      \STATE {\bfseries Input:} number of iterations $n_{\text{iter}}$, smoothing parameter $\beta$, CDC coefficient $\alpha$, learning rate $\eta$
      \STATE Initialize $\theta_Y^{(0)}, \theta_A^{(0)}, \theta_R^{(0)}$
      \FOR{$i=1$ {\bfseries to} $n_{\text{iter}}$}
      \STATE Update gradient descent $\theta_Y^{(i)} \gets \theta_Y^{(i-1)} - \eta \nabla_{\theta_Y} \big[ \mathcal{L}_{G_Y} (\theta_Y^{(i-1)}, \theta_R^{(i-1)}) \big]$
      \STATE Update gradient descent $\theta_R^{(i)} \gets \theta_R^{(i-1)} - \eta \nabla_{\theta_R} \big[ \mathcal{L}_{G_Y} (\theta_Y^{(i-1)}, \theta_R^{(i-1)}) + \alpha \mathcal{L}_{\text{conf}} (\theta_{A, \text{EMA}}^{(i-1)}, \theta_R^{(i-1)}) \big]$
      \STATE Update EMA $\theta_{Y, \text{EMA}}^{(i)} \gets \beta  \theta^{(i - 1)}_{Y, \text{EMA}} + (1 - \beta) \theta_{Y}^{(i)}$
      \STATE Update EMA $\theta_{R, \text{EMA}}^{(i)} \gets \beta  \theta^{(i - 1)}_{R, \text{EMA}} + (1 - \beta) \theta_R^{(i)}$
      \STATE Update gradient descent $\theta_A^{(i)} \gets \theta_A^{(i-1)} - \eta \nabla_{\theta_A} \big[ \mathcal{L}_{G_A} (\theta_A^{(i-1)}, \theta_{R, \text{EMA}}^{(i)}) \big]$
      \STATE Update EMA $\theta_{A, \text{EMA}}^{(i)} \gets \beta  \theta^{(i - 1)}_{A, \text{EMA}} + (1 - \beta) \theta_A^{(i)}$
      \ENDFOR
    \end{algorithmic}
\end{algorithm}

\clearpage
\section{Proof of Theorem~\ref{thrm:domain_conf_loss}}  \label{app:proof}

We begin by stating a lemma similar to \emph{Lemma 1} \cite{bica2020estimating}, yet ours includes the treatment probabilities $\mathbb{P}(\mathbf{A}_t = a_i)$ of Eq.~\eqref{eq:ga-star}.

\begin{lemma}\label{lem:sol_G}
Let $\alpha_i = \mathbb{P}(\mathbf{A}_t = a_i)$ and $x^\prime = \Phi(\bar{\mathbf{H}}_t)$ for some fixed representation network $\Phi(\cdot)$. Then, it holds that
\begin{equation}\label{eq:sol_G}
    {G^\ast}^j_A(x^\prime) =  \frac{\alpha_j P^\Phi_j(x^\prime)}{\sum_{i=1}^{d_a}\alpha_i P^\Phi_i(x^\prime)}.
\end{equation}
\end{lemma}
\begin{proof}
The objective in Eq.~\eqref{eq:obj_G} is obtained for fixed $\Phi$ by maximizing the following objective pointwise for any $x^\prime$:
\begin{equation}
     G^\ast_A = \argmax_{G_A} \sum_{j=1}^{d_a} \alpha_j \log\left({G}^j_A(x^\prime) \right) P^\Phi_j(x^\prime) \quad \text{subject to} \quad  \sum_{i=1}^{d_a} {G}^i_A(x^\prime) = 1.
\end{equation}
The result can now be obtained by applying Lagrange multipliers as done in \cite{bica2020estimating}.
\end{proof}

We now derive our theorem.

\begin{theorem}\label{thrm:domain_conf_loss}
We fix $t \in \mathbb{N}$ and define $P$ as the distribution of $\bar{\mathbf{H}}_t$, $P_j$ as the distribution of $\bar{\mathbf{H}}_t$ given $\mathbf{A}_t = a_j$, and $P^\Phi_j$ as the distribution of $\Phi(\bar{\mathbf{H}}_t)$ given $\mathbf{A}_t = a_j$ for all $j \in \{1, \dots, d_a\}$. Let $G^j_A$ denote the output of $G_A$ corresponding to treatment $a_j$. Then, there exists an optimal pair $(\Phi^\ast, G^\ast_A)$ such that
\begin{align}\
 \Phi^\ast &= \argmax_{\Phi} \sum_{j=1}^{d_a} \mathbb{E}_{\bar{\mathbf{H}}_t \sim P}\left[ \log{G^\ast}^j_A(\Phi(\bar{\mathbf{H}}_t) \right]\label{eq:obj_Phi} \\
 G^\ast_A &= \argmax_{G_A} \sum_{j=1}^{d_a} \mathbb{E}_{\bar{\mathbf{H}}_t \sim P_j}\left[\log{G}^j_A(\Phi^\ast(\bar{\mathbf{H}}_t) \right] \mathbb{P}(\mathbf{A}_t = a_j) \label{eq:obj_G}\\
& \quad \quad \text{subject to }  \sum_{i=1}^{d_a} {G}^i_A(\Phi^\ast(\bar{\mathbf{H}}_t)) = 1.
\end{align}
Furthermore, $\Phi^\ast$ satisfies Eq.~\eqref{eq:obj_Phi} if and only if it induces balanced representations across treatments, i.e., $P^{\Phi^\ast}_1 = \ldots = P^{\Phi^\ast}_{d_a}$.
\end{theorem}

\begin{proof}[Proof of Theorem \ref{thrm:domain_conf_loss}]
We plug the optimal prediction probabilities provided by Lemma~\ref{lem:sol_G} into the objective Eq.~\eqref{eq:obj_Phi} and obtain
\begin{align}
    \sum_{j=1}^{d_a} \mathbb{E}_{x^\prime \sim P^\Phi}\left[ \log\frac{\alpha_j P^\Phi_j(x^\prime)}{\sum_{i=1}^{d_a}\alpha_i P^\Phi_i(x^\prime)} \right] &=  \sum_{j=1}^{d_a} \int \log\left(\frac{ P^\Phi_j(x^\prime)}{\sum_{i=1}^{d_a}\alpha_i P^\Phi_i(x^\prime)} \right)  P^\Phi(x^\prime) \, \dd x^\prime + \underbrace{\sum_{j=1}^{d_a} \log(\alpha_j)}_{=C} \\
    &= \sum_{j=1}^{d_a} \int \log\left(\frac{ P^\Phi_j(x^\prime)}{\sum_{i=1}^{d_a}\alpha_i P^\Phi_i(x^\prime)} \right)  \sum_{i=1}^{d_a}\alpha_i P^\Phi_i(x^\prime) \, \dd x^\prime + C \\
    &= - \sum_{j=1}^{d_a} \mathrm{KL}\left(\sum_{i=1}^{d_a}\alpha_i P^\Phi_i(x^\prime) \; \Bigg\vert\Bigg\vert\; P^\Phi_j(x^\prime)\right) + C.
\end{align}

Hence, the objective becomes
\begin{equation}
 \min_{\Phi}   \sum_{j=1}^{d_a} \mathrm{KL}\left(\sum_{i=1}^{d_a}\alpha_i P^\Phi_i(x^\prime) \; \Bigg\vert\Bigg\vert\; P^\Phi_j(x^\prime)\right).
\end{equation}
For balanced representations $P^\Phi_1 = \dots = P^\Phi_{d_a}$, we obtain a global minimum because
\begin{equation}\label{eq:opt_Phi_KL}
    \mathrm{KL}\left(\sum_{i=1}^{d_a}\alpha_i P^\Phi_i(x^\prime) \; \Bigg\vert\Bigg\vert\;  P^\Phi_j(x^\prime)\right) = \mathrm{KL}\left(P^\Phi_1(x^\prime) \;\Big\vert\Big\vert\,  P^\Phi_1(x^\prime)\right) = 0
\end{equation}
for all $j \in \{1, \dots, d_a\}$.

Let us now assume that there exists an optimal $\Phi$ that satisfies Eq.~\eqref{eq:opt_Phi_KL} and that induces unbalanced representations, i.e., there exists an $j \neq \ell$ with $P^\Phi_j \neq P^\Phi_\ell$. This implies 
\begin{equation}
    \sum_{i=1}^{d_a}\alpha_i P^\Phi_i(x^\prime) = P^\Phi_j \neq P^\Phi_\ell = \sum_{i=1}^{d_a}\alpha_i P^\Phi_i(x^\prime) ,
\end{equation}
which is a contradiction. Hence, $\Phi$ attains the global optimum if and only if it induces balanced representations.
\end{proof}

\clearpage
\section{Baseline Methods}  \label{app:baselines}

We select four methods as baselines, which make use of the same setting as our work (see Sec.~\ref{sec:related-work}). These are: (1)~marginal structural models~(\textbf{MSMs}) \cite{robins2000marginal, hernan2001marginal}, (2)~recurrent marginal structural networks~(\textbf{RMSNs}) \cite{lim2018forecasting}, (3)~counterfactual recurrent network~(\textbf{CRN}) \cite{bica2020estimating}, and (4) \textbf{G-Net} \cite{li2021g}. We provide details for each in the following.

\subsection{Marginal Structural Models (MSMs)}

Marginal structural models (MSMs) \cite{robins2000marginal, hernan2001marginal} are a standard baseline from epidemiology, which aim at counterfactual outcomes estimation with inverse probability of treatment weights (IPTW) via linear modeling. Time-varying confounding bias is removed with the help of stabilized weights
\begin{equation}
    \label{eq:sw}
    SW(t, \tau) = \frac{\prod_{n=t}^{t+\tau} f\left(\mathbf{A}_n  \;\mid\;  \bar{\mathbf{A}}_{n-1}\right)}{\prod_{n=t}^{t+\tau} f\left(\mathbf{A}_n  \;\mid\;  \bar{\mathbf{H}}_{n} \right)} ,
\end{equation}
where $\tau$ ranges from $1$ to $\tau_{\text{max}}$, and $f\left(\mathbf{A}_n  \;\mid\;  \bar{\mathbf{A}}_{n-1}\right)$ and $f\left(\mathbf{A}_n  \;\mid\;  \bar{\mathbf{H}}_{n} \right)$ denote the conditional probabilities mass functions for discrete treatments of $\mathbf{A}_n$ given $\bar{\mathbf{A}}_{n-1}$ and $\bar{\mathbf{H}_{n}}$, respectively. Both are estimated with logistic regressions, which depend on the sum of previous treatment applications, two previous time-varying covariates and static covariates
\begin{align}
    f\left(\mathbf{A}_t  \;\mid\;  \bar{\mathbf{A}}_{t-1}\right) & = \sigma \left( \sum_{j=1}^{d_A} \omega_j \sum_{n=1}^{t-1} \mathbbm{1}_{[\mathbf{A}_n = a_j]} \right), \\
    \label{eq:msm-history}
    f\left(\mathbf{A}_t  \;\mid\;  \bar{\mathbf{H}}_{t} \right) & = \sigma \left(W_{1,x} \textbf{X}_t + W_{2,x} \textbf{X}_{t-1} + W_{1,y} \textbf{Y}_t + W_{2,y} \textbf{X}_{t-1} + W_{v} \textbf{V} +   \sum_{j=1}^{d_A} \phi_j \sum_{n=1}^{t-1} \mathbbm{1}_{[\mathbf{A}_n = a_j]} \right),
\end{align}
where $\sigma(\cdot)$ is a sigmoid function and where $\omega_{\cdot}, \phi_{\cdot}, W_{\cdot}$ are logistic regression parameters. After the stabilized weights are estimated, they are normalized and truncated at their 1-st and 99-th percentiles as done in \cite{lim2018forecasting}.

Counterfactual outcome regressions are fit for each prediction horizon $\tau$ separately. For a specific $\tau$, we split dataset into smaller chunks with a rolling origin and calculate stabilized weights for each chunk. Outcome regressions use the same history inputs, as $f\left(\mathbf{A}_n  \;\mid\;  \bar{\mathbf{H}}_{n} \right)$ (Eq.~\eqref{eq:msm-history}). 

MSMs do not contain hyperparameters; thus, we have merged train and validation subsets for all the experiments.

\subsection{Recurrent Marginal Structural Networks (RMSNs)}

RMSNs refer to sequence-to-sequence architectures consisting of four LSTM subnetworks: propensity treatment network, propensity history network, encoder, and decoder. RMSNs are designed to handle multiple binary treatments. The encoder first learns a representation of the observed history $\bar{\mathbf{H}}_t$ to perform one-step-ahead prediction. The decoder then uses this representation for estimating $\tau$-step-ahead counterfactual outcomes. A fully-connected linear layer (memory adapter) is used to match the size of the representation  of the encoder and the hidden units of the decoder. 

In RMSNs, time-varying confounding is addressed by re-weighting the objective with the IPTW \cite{robins2000marginal} during training. IPTW creates a pseudo-population that mimics a randomized controlled trial. As done in \cite{lim2018forecasting}, we use the stabilized weights (Eq.~\eqref{eq:sw}). Both $f\left(\mathbf{A}_n  \;\mid\;  \bar{\mathbf{A}}_{n-1}\right)$ and $f\left(\mathbf{A}_n  \;\mid\;  \bar{\mathbf{H}}_{n} \right)$ are learned from the data using LSTM networks, which are called propensity treatment network (nominator) and propensity history network (denominator).

During training, the propensity networks are trained first to estimate the stabilized weights $SW(t, \tau)$. Afterward, the encoder is trained using a mean squared error~(MSE) weighted with $SW( \cdot, 1 )$. Similarly to MSMs, stabilized weights are normalized and truncated.

Finally, the decoder is trained by minimizing the loss using the full stabilized weights $SW( \cdot, \tau_{\text{max}})$. For this purpose, the dataset is processed into smaller chunks with rolling origins, and, for each rolling origin, a representation is built using the trained encoder. We refer to \cite{lim2018forecasting} for details on the training algorithm.

We tuned the same hyperparameters, as in the original paper \cite{lim2018forecasting} (see details in Appendix~\ref{app:hparams}).

\subsection{Counterfactual Recurrent Network (CRN)}

CRN consists of an encoder-decoder architecture. In contrast to RMSNs, which use IPTW to address time-varying confounding, CRN builds balanced representations which are non-predictive of the treatment assignment. This is achieved by adopting an adversarial learning technique, namely gradient reversal \cite{ganin2015unsupervised}.

In CRN, both encoder and decoder consist of a single LSTM-layer. Unlike RMSNs, the authors and we did not use a memory adapter. Thus, the number of LSTM hidden units $d_h$ of decoder is set to the size of the balanced representation of the encoder.

At each time step $t$, the hidden states $\mathbf{h}_t$ are fed into a fully-connected linear layer that builds a representation $\mathbf{\Phi}_t$. Then, two fully-connected networks $G_Y$ and $G_A$, put on top of $\mathbf{\Phi}_t$, aim to predict the next outcome $\mathbf{Y}_{t+1}$ and the current treatment $\mathbf{A}_t$, respectively. For this, both encoder and decoder are trained by minimizing the loss
\begin{equation}
    \mathcal{L} = \left\Vert \mathbf{Y}_{t+1} - G_Y\big(\mathbf{\Phi}_t, \mathbf{A}_t \big) \right\Vert^2 - \lambda \, \sum_{j=1}^{d_a} \mathbbm{1}_{[\mathbf{A}_t = a_j]} \log G_A (\mathbf{\Phi}_t)
\end{equation}
with hyperparameter $\lambda$. The loss $\mathcal{L}$ is based on a gradient reversal layer \cite{ganin2015unsupervised}, which forces $G_A$ to minimize cross-entropy between predicted and current treatment, but $\mathbf{\Phi}_t$ to maximize it. In our experiments, we kept $\lambda = 1$, as it was used by \cite{atan2018learning,bica2020estimating}.

In our ablation study (Sec.~\ref{sec:ablation}), we combined CRN with our CDC loss. For that, we applied our adversarial training procedure (introduced in Sec.~\ref{sub-sec:CT-training}) to representations of LSTM-based encoder and decoder, and feed-forward networks $G_Y$ and $G_A$. Here, EMA of model parameters ($\beta = 0.99$) was also accompanying the CDC loss.

\subsection{G-Net}

G-Net is based the G-computation formula from Eq.~\eqref{eq:gformula}, which expresses the average counterfactual outcome $\mathbf{Y}_{t + \tau}[\bar{\mathbf{a}}_{t, t+\tau-1}]$ conditioned on the history $\bar{\mathbf{H}}_t$ in terms of the observational data distribution.

G-Net performs counterfactual outcomes prediction in two steps: First, the conditional distributions $\mathbb{P}(\mathbf{X}_{j} \mid \bar{\mathbf{H}}_t, \bar{\mathbf{x}}_{t+1:j-1}, \bar{\mathbf{a}}_{t:j-1})$ are estimated. Then, Monte Carlo simulations are performed via Eq.~\eqref{eq:gformula}, by sampling from the estimated distributions. Afterward, $\mathbf{Y}_{t + \tau}[\bar{\mathbf{a}}_{t, t+\tau-1}]$ is predicted by taking the empirical mean over the Monte Carlo samples ($M = 50$ in our experiments). 

The conditional distributions $\mathbb{P}(\mathbf{X}_{j} \mid \bar{\mathbf{H}}_t, \bar{\mathbf{x}}_{t+1:j-1}, \bar{\mathbf{a}}_{t:j-1})$ are learned by estimating the respective conditional expectations $\mathbb{E}(\mathbf{X}_{j} \mid \bar{\mathbf{H}}_t, \bar{\mathbf{x}}_{t+1:j-1}, \bar{\mathbf{a}}_{t:j-1})$, which are learned via a single LSTM jointly with outcome prediction. One can then sample from $\mathbb{P}(\mathbf{X}_{t+j} \mid \bar{\mathbf{H}}_t, \bar{\mathbf{x}}_{t+1:j-1}, \bar{\mathbf{a}}_{t:j-1})$ by drawing from the empirical distributions of the residuals on some holdout set not used to estimate the conditional expectations. We used 10\% of the training data for the holdout dataset.

For better comparability with other baselines, we used one or two-layered LSTMs (as in the original papers) with an extra fully-connected linear representation layer and a network with hidden units on top of the latter (analogous to $G_Y$ in \shortname or CRN).

\clearpage
\section{Hyperparameter Tuning} \label{app:hparams}

We performed hyperparameter optimization for all benchmarks via random grid search with respect to the factual RMSE of the validation set. We list the ranges of hyperparameter grids in Table~\ref{tab:data-specific-hyperparams}. We report additional information on model-specific hyperparameters in Table~\ref{tab:common-hyperparams} (here we used the same ranges for all experiments). For reproducibility, we make the selected hyperparameters public: they can be found in YAML format in our GitHub\footnote{ \url{https://github.com/Valentyn1997/CausalTransformer/}}. 

We aimed for a fair comparison and thus kept the number of parameters and layers similar across datasets and models. Nevertheless, the hyperparameter ranges differ slightly for each dataset and model, as the size of inputs is different (see Table \ref{tab:common-hyperparams}). Thus, e.g., the range of sizes of hidden units (sequential, representational, or fully-connected) is decreased for the MIMIC-III-based experiments. In specific cases (LSTM hidden units propensity treatment network of RMSNs or transformer units of \shortname), we discarded unrealistically small values for synthetic datasets. For the fully-synthetic dataset based on the tumor growth simulator, we use one layer sequential models. For MIMIC-III, we also include two-layered LSTMs/transformers. The number of epochs ($n_e$) is also chosen differently for each dataset to reflect its complexity. \shortname generally requires more epochs to converge due to the EMA of model weights. Therefore, we used approximately 60\,\% more epochs for \shortname than other models. Note that \shortname still outperforms the baselines when EMA is omitted, as shown in our ablation study. Due to the high memory usage of self-attention for long sequences and batch augmentation with masked vitals of \shortname, we also use smaller ranges of minibatch sizes for \shortname.  Notably, as in \shortname, we omitted the final projection layer after concatenation of the attention heads, as we need the size of hidden units (which always depends on the input size while tuning) to always be divisible by the number of heads $n_h$. Thus, we have chosen the closest larger divisible by the number of hidden units. 

\textbf{Training of baselines:} All baseline models are implemented in PyTorch Lightning and, as our \shortname, trained with Adam \cite{kingma2014adam}. The number of epochs ($n_e$) is varied across datasets for a better fit. 

We perform exponential rise of both $\alpha$ (in the CDC loss) and $\lambda$ (in gradient reversal). This is given by
\begin{equation}
    \alpha_e = \alpha \cdot \Big(\frac{2}{1 + \exp (-10 * e / n_e)} - 1 \Big), \quad \quad \lambda_e = \lambda \cdot \Big(\frac{2}{1 + \exp (-10 * e / n_e)} - 1 \Big) ,
\end{equation}
where $e \in 1, \ldots, n_e$ is an index of current epoch.

For all baselines, we  also used the teacher forcing technique \cite{williams1989learning} when training the models for multiple-step-ahead prediction. During evaluation of multiple-step-ahead prediction, we switch off teacher forcing and autoregressively feed model predictions. 

\begin{table*}[t]
    \caption{Ranges for hyperparameter tuning across experiments. Here, we distinguish (1)~data using the tumor growth (TG) simulator ($=$experiments with fully-synthetic data), (2)~data from the semi-synthetic (SS) benchmark, and (3)~real-world (RW) MIMIC-III data. C is the input size. $d_r$ is the size of balanced representation (BR) or the output of LSTM (in the case of G-Net).}
    \label{tab:data-specific-hyperparams}
    \begin{center}
        \footnotesize
                \begin{tabular}{l|l|l|c|c|c}
                    \toprule
                    Model & Sub-model & Hyperparameter & Range (TG simulator) & Range (SS data) & Range (RW data) \\
                    \midrule
                    \multirow{21}{*}{RMSNs} 
                    & \multirow{7}{*}{\begin{tabular}{l} Propensity \\ treatment \\ network \end{tabular}} 
                       & LSTM layers ($B$) & 1 & 1, 2 & 1, 2 \\
                       && Learning rate ($\eta$) & \multicolumn{3}{c}{0.01, 0.001, 0.0001} \\
                       && Minibatch size & \multicolumn{3}{c}{64, 128, 256} \\
                       && LSTM hidden units ($d_h$) & 0.5C, 1C, 2C, 3C, 4C & 0.5C, 1C, 2C & 0.5C, 1C, 2C\\
                       && LSTM dropout rate ($p$) &  \multicolumn{3}{c}{0.1, 0.2, 0.3, 0.4, 0.5} \\
                       && Max gradient norm &  \multicolumn{3}{c}{0.5, 1.0, 2.0} \\
                       && Number of epochs ($n_e$) &  100 & 400 & 200 \\
                    \cmidrule{2-6}
                    & \multirow{7}{*}{\begin{tabular}{l} Propensity \\ history \\ network \\ \midrule  Encoder \end{tabular}} 
                       & LSTM layers ($B$) & 1 & 1, 2 & 1, 2 \\
                       && Learning rate ($\eta$) & \multicolumn{3}{c}{0.01, 0.001, 0.0001} \\
                       && Minibatch size & \multicolumn{3}{c}{64, 128, 256} \\
                       && LSTM hidden units ($d_h$) & 0.5C, 1C, 2C, 3C, 4C  & 0.5C, 1C, 2C & 0.5C, 1C, 2C  \\
                       && LSTM dropout rate ($p$) &  \multicolumn{3}{c}{0.1, 0.2, 0.3, 0.4, 0.5} \\
                       && Max gradient norm &  \multicolumn{3}{c}{0.5, 1.0, 2.0} \\
                       && Number of epochs ($n_e$) &  100 & 400 & 200 \\
                    \cmidrule{2-6}
                    & \multirow{7}{*}{\begin{tabular}{l} Decoder \end{tabular}} 
                       & LSTM layers ($B$) & 1 & 1, 2 & 1, 2 \\
                       && Learning rate ($\eta$) & \multicolumn{3}{c}{0.01, 0.001, 0.0001} \\
                       && Minibatch size & \multicolumn{3}{c}{256, 512, 1024} \\
                       && LSTM hidden units ($d_h$) & 1C, 2C, 4C, 8C, 16C & 1C, 2C, 4C & 1C, 2C, 4C\\
                       && LSTM dropout rate ($p$) &  \multicolumn{3}{c}{0.1, 0.2, 0.3, 0.4, 0.5} \\
                       && Max gradient norm &  \multicolumn{3}{c}{0.5, 1.0, 2.0, 4.0} \\
                       && Number of epochs ($n_e$) & 100 & 400 & 200 \\
                    \midrule
                    \multirow{16}{*}{CRN} & \multirow{8}{*}{\begin{tabular}{l} Encoder \end{tabular}} 
                       & LSTM layers ($B$) & 1 & 1, 2 & 1, 2\\
                       && Learning rate ($\eta$) & \multicolumn{3}{c}{0.01, 0.001, 0.0001}\\
                       && Minibatch size & \multicolumn{3}{c}{64, 128, 256} \\
                       && LSTM hidden units ($d_h$) & 0.5C, 1C, 2C, 3C, 4C & 0.5C, 1C, 2C & 0.5C, 1C, 2C\\
                       && BR size ($d_r$) & 0.5C, 1C, 2C, 3C, 4C &  0.5C, 1C, 2C &  0.5C, 1C, 2C \\
                       && FC hidden units ($n_{\text{FC}}$) & 0.5$d_r$, 1$d_r$, 2$d_r$, 3$d_r$, 4$d_r$ & 0.5$d_r$, 1$d_r$, 2$d_r$ & 0.5$d_r$, 1$d_r$, 2$d_r$ \\
                       && LSTM dropout rate ($p$) &  \multicolumn{3}{c}{0.1, 0.2, 0.3, 0.4, 0.5} \\
                       && Number of epochs ($n_e$) &  100 & 400 & 200 \\
                    \cmidrule{2-6}
                    & \multirow{8}{*}{\begin{tabular}{l} Decoder \end{tabular}} 
                       & LSTM layers ($B$) & 1 & 1, 2 & 1, 2 \\
                       && Learning rate ($\eta$) & \multicolumn{3}{c}{0.01, 0.001, 0.0001} \\
                       && Minibatch size & \multicolumn{3}{c}{256, 512, 1024} \\
                       && LSTM hidden units ($d_h$) & \multicolumn{3}{c}{BR size of encoder} \\
                       && BR size ($d_r$) & 0.5C, 1C, 2C, 3C, 4C & 0.5C, 1C, 2C & 0.5C, 1C, 2C \\
                       && FC hidden units ($n_{\text{FC}}$) & 0.5$d_r$, 1$d_r$, 2$d_r$, 3$d_r$, 4$d_r$ & 0.5$d_r$, 1$d_r$, 2$d_r$ & 0.5$d_r$, 1$d_r$, 2$d_r$ \\
                       && LSTM dropout rate ($p$) &  \multicolumn{3}{c}{0.1, 0.2, 0.3, 0.4, 0.5} \\
                       && Number of epochs ($n_e$) &  100 & 400 & 200 \\
                    \midrule
                    \multirow{8}{*}{G-Net} & \multirow{8}{*}{\begin{tabular}{l} --- \end{tabular}}
                        & LSTM layers ($B$) & 1 & 1, 2 & 1, 2 \\
                        && Learning rate ($\eta$) & \multicolumn{3}{c}{0.01, 0.001, 0.0001} \\
                        && Minibatch size & \multicolumn{3}{c}{64, 128, 256} \\
                        && LSTM hidden units ($d_h$) & 0.5C, 1C, 2C, 3C, 4C & 0.5C, 1C, 2C & 0.5C, 1C, 2C \\
                        && LSTM output size ($d_r$) & 0.5C, 1C, 2C, 3C, 4C & 0.5C, 1C, 2C & 0.5C, 1C, 2C \\
                        && FC hidden units ($n_{\text{FC}}$) & 0.5$d_r$, 1$d_r$, 2$d_r$, 3$d_r$, 4$d_r$ & 0.5$d_r$, 1$d_r$, 2$d_r$ & 0.5$d_r$, 1$d_r$, 2$d_r$ \\
                        && LSTM dropout rate ($p$) &  \multicolumn{3}{c}{0.1, 0.2, 0.3, 0.4, 0.5} \\
                        && Number of epochs ($n_e$) & 50 & 400 & 200 \\
                    \midrule
                    \multirow{10}{*}{\shortname} & \multirow{10}{*}{\begin{tabular}{l} --- \end{tabular}}
                        & Transformer blocks ($B$) & 1 & 1, 2 & 1, 2 \\
                        && Learning rate ($\eta$) & \multicolumn{3}{c}{0.01, 0.001, 0.0001} \\
                        && Minibatch size & 64, 128, 256 & 32, 64 & 32, 64 \\
                        && Attention heads ($n_h$) & 2 & 2, 3 & 2, 3 \\
                        && Transformer units ($d_h$) & 1C, 2C, 3C, 4C & 0.5C, 1C, 2C & 0.5C, 1C, 2C \\
                        && BR size ($d_r$) & 0.5C, 1C, 2C, 3C, 4C & 0.5C, 1C, 2C & 0.5C, 1C, 2C \\
                        && FC hidden units ($n_{\text{FC}}$) & 0.5$d_r$, 1$d_r$, 2$d_r$, 3$d_r$, 4$d_r$ & 0.5$d_r$, 1$d_r$, 2$d_r$ & 0.5$d_r$, 1$d_r$, 2$d_r$ \\
                        && Sequential dropout rate ($p$) &  \multicolumn{3}{c}{0.1, 0.2, 0.3, 0.4, 0.5} \\
                        && Max positional encoding ($l_{\text{max}}$) & 15 & 20 & 30 \\
                        && Number of epochs ($n_e$) & 150 & 400 & 300 \\
                    \bottomrule
                \end{tabular}
    \end{center}
\end{table*}

\begin{table*}[t]
    \caption{Additional information on model-specific hyperparameters (kept the same for all experiments).}
    \label{tab:common-hyperparams}
    \begin{center}
        \begin{small}
                \begin{tabular}{l|l|l|l}
                    \toprule
                    Model & Sub-model & Hyperparameter & Value \\
                    \midrule
                    \multirow{12}{*}{RMSNs} 
                    & \multirow{3}{*}{\begin{tabular}{l} Propensity treatment network \end{tabular}} 
                        & Random search iterations & 50\\
                       && Input size (C) & $d_a$ \\
                       && Output size & $d_a$ \\
                    \cmidrule{2-4}
                     & \multirow{3}{*}{\begin{tabular}{l} Propensity history network \end{tabular}} 
                        & Random search iterations & 50\\
                       && Input size (C) & $d_a + d_y + d_x + d_v$ \\
                       && Output size & $d_a$ \\
                    \cmidrule{2-4}
                    & \multirow{3}{*}{\begin{tabular}{l}Encoder \end{tabular}} 
                        & Random search iterations & 50\\
                       && Input size (C) & $d_a + d_y + d_x + d_v$ \\
                       && Output size & $d_y$ \\
                    \cmidrule{2-4}
                    & \multirow{3}{*}{\begin{tabular}{l} Decoder \end{tabular}} 
                        & Random search iterations & 20\\
                       && Input size (C) & $d_a + d_y + d_v$ \\
                       && Output size & $d_y$ \\
                    \midrule
                    \multirow{8}{*}{CRN} 
                    & \multirow{4}{*}{\begin{tabular}{l} Encoder \end{tabular}} 
                        & Random search iterations & 50\\
                        && Input size (C) & $d_a + d_y + d_x + d_v$ \\
                        && Output size & $d_a + d_y$ \\
                        && Gradient reversal coefficient ($\lambda$) & 1.0 \\
                    \cmidrule{2-4}
                    & \multirow{4}{*}{\begin{tabular}{l} Decoder \end{tabular}} 
                        & Random search iterations & 30 \\
                       && Input size (C) & $d_a + d_y + d_v$ \\
                       && Output size & $d_a + d_y$ \\
                       && Gradient reversal coefficient ($\lambda$) & 1.0 \\
                    \midrule
                    \multirow{6}{*}{G-Net} & \multirow{6}{*}{\begin{tabular}{l} ---\end{tabular}}
                        & Random search iterations & 50\\
                        && Input size (C) & $d_a + d_y + d_x + d_v$ \\
                        && Output size & $d_y + d_x$ \\
                        && MC samples ($M$) & 50 \\
                        && Number of covariate groups & 1 \\
                        && Holdout dataset ratio (empirical residuals) & 10\% \\
                    \midrule
                    \multirow{6}{*}{\shortname} & \multirow{6}{*}{\begin{tabular}{l} --- \end{tabular}}
                        & Random search iterations & 50\\
                        && Input size (C) & $\max\{d_a, d_y, d_x, d_v\}$ \\
                        && Output size & $d_a + d_y$ \\
                        && CDC coefficient ($\alpha$) & 0.01 \\
                        && EMA of model weights ($\beta$) & 0.99 \\
                        && Positional encoding & relative, trainable \\
                    \bottomrule
                \end{tabular}
        \end{small}
    \end{center}
    \vskip -0.1in
\end{table*}

\clearpage
\section{Encoder-Decoder Causal Transformer} \label{app:EDCT}
\subsection{Overview}

Here, we summarize the \emph{Encoder-Decoder} \longname~(EDCT) from our ablation study, namely a single-subnetwork version of out \shortname. The EDCT consists of transformer-based encoder and decoder (see Fig.~\ref{fig:EDCT}). The encoder builds a treatment-invariant sequence of representations of the history $\bar{\mathbf{\Phi}}_t = (\mathbf{\Phi}_1, \dots, \mathbf{\Phi}_t )$, balanced with a custom adversarial objective. The decoder subsequently uses $\bar{\mathbf{\Phi}}_t$ as cross-attention keys and values for estimating outcomes of future treatments.

We start by mapping the concatenated time-varying covariates, left-shifted treatments, outcomes and static covariates to a hidden state space of dimensionality $d_h$ via fully-connected linear layer:
\begin{equation}
    \mathbf{h}_t^0 = \operatorname{Linear}( \operatorname{Concat}(\mathbf{A}_{t-1}, \mathbf{Y}_{t},  \mathbf{X}_{t},  \mathbf{V})).
\end{equation}

In the case of decoder, we apply a similar input transformation
\begin{equation}
    \mathbf{h}_t^0 = \operatorname{Linear}( \operatorname{Concat}(\mathbf{a}_{t-1}, \hat{\mathbf{Y}}_{t}, \mathbf{V})),
\end{equation}
where $\hat{\mathbf{Y}}_{t}$ are autoregressively-fed model outputs. 

We then stack of $B$ identical encoder/decoder blocks or layers, which transform the whole sequence of hidden states $\big(\mathbf{h}_1^0, \dots, \mathbf{h}_t^0\big)$ in quadratic time, depending on sequence length $t$. This is given by
\begin{align}
    & \mathrm{H}^b = \big(\mathbf{h}_1^b, \dots, \mathbf{h}_t^b\big)^\top \in \mathbb{R}^{t \times d_h} , \label{eq:seq-concat}\\
    & \mathrm{H}^b = \text{Block}_b(\mathrm{H}^{b - 1}), \qquad b \in \{1, \dots, B\} ,
\end{align}
where $B$ is the total number of blocks.

The encoder uses its hidden states  to infer keys, queries, and values (thus: self-attention). In contrast, the decoder has both self- and cross-attentions. For later, we use keys and values, inferred from the sequence of balanced representations of the history. Note that the dimensionality of hidden decoder state is set such that it matches the size of the balanced representations of the encoder, i.e., $d_h = d_r$. 

\begin{figure*}[h!]
    \vskip 0.2in
    \begin{center}
    \centerline{\includegraphics[width=\textwidth]{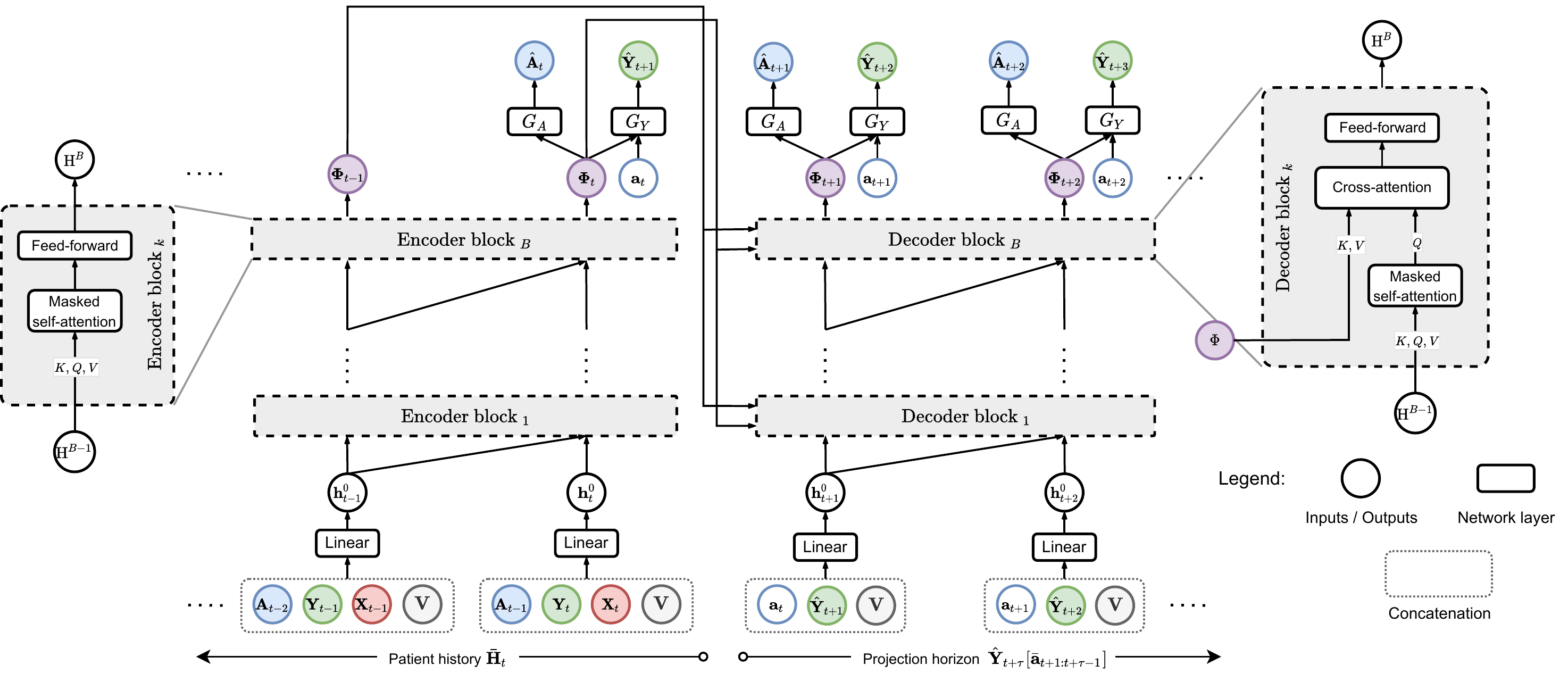}}
    \caption{Architecture of the encoder-decoder causal transformer (EDCT). Residual connections with layer normalizations are omitted for clarity. The encoder is trained to perform one-step-ahead prediction $\hat{\mathbf{Y}}_{t+1}[\mathbf{a}_t]$, whereas the decoder uses the pretrained balanced representations of history from the encoder. Based on them, the decoders makes predictions for the projection horizon $\tau \ge 2$ via  $\hat{\mathbf{Y}}_{t+\tau}[\bar{\mathbf{a}}_{t+1: t+\tau-1}]$.}
    \label{fig:EDCT}
    \end{center}
\end{figure*}

Lastly, we take the balanced representations from the last transformer block as outputs. Here, we apply an additional fully-connected linear layer and exponential linear unit (ELU) non-linearity as in Eq.~\eqref{eq:output-repr}.

We make slightly adaptations to the relative positional encoding for the cross-attention of the decoder (Eq.~\eqref{eq:cross-attention}). The decoder works a~priori with the continuation of the same sequence. However, it is here beneficial to use the sequence for construction of $a^V_{ij}$ and $a^K_{ij}$, so that $j \in \{1, \dots, t\}$ and $i \in \{t + 1, \dots, t + \tau - 1\}$. Notably, relative positional encodings are shared neither between encoder and decoder, nor between self-attention and the cross-attention of the decoder.

We show the EDCT in Fig.~\ref{fig:EDCT}. We further formalize the encoder/decoder transformer blocks in the following section. 

\clearpage
\subsection{Transformer Blocks in EDCT} 
\label{app:EDCT-blocks}

The EDCT encoder block is defined in a following way:
\begin{align}
        \tilde{\mathrm{H}}^{b - 1} &= \operatorname{LN}\Big(\operatorname{MHA}\big(Q(\mathrm{H}^{b - 1}), K(\mathrm{H}^{b - 1}), V(\mathrm{H}^{b - 1})\big) + \mathrm{H}^{b - 1}\Big) ,
    \label{eq:block-multi-head} \\ 
     \mathrm{H}^{b} &= \operatorname{LN}\big(\operatorname{FF}(\tilde{\mathrm{H}}^{b - 1}) + \tilde{\mathrm{H}}^{b - 1}\big) . \label{eq:block-ff}
\end{align}

The EDCT decoder block adds a cross-attention layer after the self-attention (i.e., between Eq.~\eqref{eq:block-multi-head} and Eq.~\eqref{eq:block-ff}). This is formalized by 
\begin{align}
        \tilde{\mathrm{H}}^{b - 1} &= \operatorname{LN}\Big(\operatorname{MHA}\big(Q(\mathrm{H}^{b - 1}), K(\mathrm{H}^{b - 1}), V(\mathrm{H}^{b - 1})\big) + \mathrm{H}^{b - 1}\Big) ,
    \\
        \tilde{\tilde{\mathrm{H}}}^{b - 1} &=  \operatorname{LN}\Big(\operatorname{MHA}\big(Q(\tilde{\mathrm{H}}^{b - 1}), K((\bar{\mathbf{\Phi}}_t)^\top), V((\bar{\mathbf{\Phi}}_t)^\top)\big) + \tilde{\mathrm{H}}^{b - 1}\Big) , \label{eq:cross-attention} 
\\
    \mathrm{H}^{b} &= \operatorname{LN}\big(\operatorname{FF}(\tilde{\tilde{\mathrm{H}}}^{b - 1}) + \tilde{\tilde{\mathrm{H}}}^{b - 1}\big) ,
\end{align}
where $\bar{\mathbf{\Phi}}_t$ is a sequence of the encoder representations (i.e., the encoded history $\bar{\mathbf{H}}_t$), as transformed according to Eq.~\eqref{eq:seq-concat}).   

\subsection{Hyperparameter Tuning for EDCT} 
\label{app:EDCT-hyparams}

We performed a hyperparameter selection for EDCT in a similar manner to CRN, see Appendix~\ref{app:hparams}. We provide hyperparameter ranges for both encoder and decoder in Table~\ref{tab:edct-data-specific-hyperparams}. Other hyperparameters are kept fixed, the same as for CRN in Table~\ref{tab:common-hyperparams}. For the fully-synthetic dataset based on the tumor growth simulator, we add a two-layer (two-block) architecture to the search range. This was done to keep the number of total trainable parameters similar to other baselines. We employed trainable relative positional encodings. Notably, the decoder has $l_{\text{max}}$ set to $\tau_{\text{max}}$ for the self-attention and to $l_{\text{max}}$ of the encoder for the cross-attention.

\begin{table*}[ht]
    \caption{Ranges for hyperparameter tuning for EDCT across experiments. Here, we distinguish (1)~data using the tumor growth (TG) simulator ($=$experiments with fully-synthetic data) and (2)~semi-synthetic (SS) data. C is the input size. $d_r$ is the size of balanced representation (BR) or the output of LSTM (in the case of G-Net).}
    \label{tab:edct-data-specific-hyperparams}
    \begin{center}
        \footnotesize
                \begin{tabular}{l|l|l|c|c}
                    \toprule
                    Model & Sub-model & Hyperparameter & Range (TG simulator) & Range (SS data) \\
                    \midrule
                    \multirow{21}{*}{EDCT} & \multirow{10}{*}{\begin{tabular}{l} Encoder \end{tabular}} 
                       & Transformer blocks ($B$) & 1, 2 & 1, 2 \\
                       && Learning rate ($\eta$) & \multicolumn{2}{c}{0.01, 0.001, 0.0001}\\
                       && Minibatch size & 64, 128, 256 & 32, 64, 128 \\
                       && Attention heads ($n_h$) & 2 & 2, 3 \\
                       && Transformer units ($d_h$) & 1C, 2C, 3C, 4C & 0.5C, 1C, 2C\\
                       && BR size ($d_r$) & 0.5C, 1C, 2C, 3C, 4C &  0.5C, 1C, 2C \\
                       && FC hidden units ($n_{\text{FC}}$) & 0.5$d_r$, 1$d_r$, 2$d_r$, 3$d_r$, 4$d_r$ & 0.5$d_r$, 1$d_r$, 2$d_r$ \\
                       && Sequential dropout rate ($p$) &  \multicolumn{2}{c}{0.1, 0.2, 0.3, 0.4, 0.5} \\
                       && Max positional encoding (self-attention) ($l_{\text{max}}$) & 15 & 20 \\
                       && Number of epochs ($n_e$) &  100 & 400 \\
                    \cmidrule{2-5}
                    & \multirow{11}{*}{\begin{tabular}{l} Decoder \end{tabular}}
                       & Transformer blocks ($B$) & 1, 2 & 1, 2 \\
                       && Learning rate ($\eta$) & \multicolumn{2}{c}{0.01, 0.001, 0.0001} \\
                       && Minibatch size & 256, 512, 1024 & 128, 256, 512 \\
                       && Attention heads ($n_h$) & 2 & 2, 3 \\
                       && Transformer units ($d_h$) & \multicolumn{2}{c}{BR size of encoder} \\
                       && BR size ($d_r$) & 0.5C, 1C, 2C, 3C, 4C & 0.5C, 1C, 2C \\
                       && FC hidden units ($n_{\text{FC}}$) & 0.5$d_r$, 1$d_r$, 2$d_r$, 3$d_r$, 4$d_r$ & 0.5$d_r$, 1$d_r$, 2$d_r$ \\
                       && Sequential dropout rate ($p$) &  \multicolumn{2}{c}{0.1, 0.2, 0.3, 0.4, 0.5} \\
                       && Max positional encoding (self-attention) ($l_{\text{max}}$) & \multicolumn{2}{c}{$\tau_{\text{max}}$} \\
                       && Max positional encoding (cross-attention) ($l_{\text{max}}$) & 15 & 20 \\
                       && Number of epochs ($n_e$) &  100 & 400 \\
                    \bottomrule
                \end{tabular}
    \end{center}
\end{table*}

\clearpage
\section{Details on Experiments with Synthetic Data} 
\label{app:syn}

\subsection{Summary of Tumor Growth Simulator} 

The tumor growth (TG) simulator \cite{geng2017prediction} models the volume of tumor $\mathbf{Y}_{t+1}$ for $t+1$ days after cancer diagnosis (so that the outcome is one-dimensional, \ie, $d_y = 1$). The model has two binary treatments: (i)~radiotherapy ($\mathbf{A}^r_t$) and (ii)~chemotherapy ($\mathbf{A}^c_t$). These are modeled as follows: (i)~Radiotherapy when assigned to a patient has an immediate effect $d(t)$ on the next outcome. (ii)~Chemotherapy affects several future outcomes with exponentially decaying effect $C(t)$ via
\begin{align}
    \mathbf{Y}_{t+1} &= \bigg(1 + \rho \log \Big(\frac{K}{\mathbf{Y}_{t}}\Big) - \beta_c C_t - (\alpha_r d_t + \beta_r d_t^2) + \varepsilon_t \bigg) \mathbf{Y}_{t} ,
\end{align}
where $\rho, K, \beta_c, \alpha_r, \beta_r$ are parameters in the simulation and where $\varepsilon_t \sim N(0, 0.01^2)$ is independently sampled noise. Here, the parameters $\beta_c, \alpha_r, \beta_r$ describe the individual response of each patient and are sampled from a mixture of truncated normal distributions with three mixture components. For exact values of parameters, refer to the code implementation\footnote{Code is available online: \url{https://github.com/Valentyn1997/CausalTransformer/blob/main/src/data/cancer_sim/cancer_simulation.py}}. The indices of mixture components are considered as static covariates ($d_v = 1$). Time-varying confounding is introduced by a biased treatments assignment, identical for both treatments; \ie,
\begin{align}
    & \mathbf{A}^c_t, \mathbf{A}^r_t \sim \operatorname{Bernoulli}\bigg( \sigma \Big( \frac{\gamma}{D_{\text{max}}}(\bar{D}_{15}(\bar{\mathbf{Y}}_{t-1}) - D_{\text{max}} / 2) \Big) \bigg) ,
\end{align}
where $\sigma(\cdot)$ is a sigmoid activation, $D_{\text{max}}$ is the maximum tumor diameter, $\bar{D}_{15}(\bar{\mathbf{Y}}_{t-1})$ is the average tumor diameter over the last 15 days, and $\gamma$ is a confounding parameter. We can control the level of confounding via $\gamma$. For $\gamma=0$, the treatment assignment is fully randomized. For increasing values, the the amount of time-varying confounding becomes also larger.

In our implementation, we proceed as follows. For RMSNs, we insert two binary treatments directly. For all other methods, we use a single-categorical treatment out of the set $\{(\mathbf{A}^c_t=0, \mathbf{A}^r_t=0)$, $(\mathbf{A}^c_t=1, \mathbf{A}^r_t=0)$, $(\mathbf{A}^c_t=0, \mathbf{A}^r_t=1)$, $(\mathbf{A}^c_t=1, \mathbf{A}^r_t=1 )\}$.

For each patient in the test set and each time step, we simulate several counterfactual trajectories, depending on $\tau$. For one-step-ahead prediction, we simulate all four combinations of one-step-ahead counterfactual outcomes $\mathbf{Y}_{t+1}$. This corresponds to the tumor volume under all possible combinations of treatment assignments. For multi-step-ahead prediction, the number of all potential outcomes of $\mathbf{Y}_{t+2}, \dots, \mathbf{Y}_{t+\tau_{\text{max}}}$ growths exponentially with the projection horizon $\tau_{\text{max}}$. Therefore, we adopt two alternative schemes:
\begin{enumerate}
\item \emph{Single sliding treatment}. To test that the correct timing of a treatment is chosen, we simulate trajectories with a single treatment but where the treatments are iteratively moved over a window ranging from $t$ to $t + \tau_{\text{max}} - 1$. This effectively results in $2(\tau_{\text{max}} - 1)$ trajectories. 
\item \emph{Random trajectories}. Here, we simulate a fixed number of trajectories, \ie, $2\,(\tau_{\text{max}} - 1)$, each with random treatment assignments. 
\end{enumerate} 
The former setting is identical to the one in \cite{bica2020estimating}. We additionally included the latter setting, as it may also involve more diverse trajectories with multiple treatments. Thereby, we hope to make our analysis more realistic with respect to clinical practice. 

For each level of confounding $\gamma$, we simulate 10,000 patient trajectories for training, 1,000 for validation, and 1,000 trajectories for testing. We limit the length for trajectories to max. 60 time steps (some patients have shorter trajectories due to recovery or death). Here, and in all following experiments, we apply hyperparameter tuning.

\subsection{Experimental Details}

\textbf{Hyperparameter tuning.} We perform hyperparameter tuning separately for all models as well as all the different values of the confounding amount $\gamma$. For this, we use the 1,000 factual patient time-series from the validation set. Details are in \ref{app:hparams}.

\textbf{Performance measurement:} We retrain the models on five simulated datasets with different random seeds. We then report the averaged root mean square error (RMSE) on the test set, that is, for hold-out data. We report a normalized RMSE, where we normalize by the maximum tumor volume $V_{\text{max}} = 1150\,\text{cm}^3$. 

We acknowledge that our results are slightly different from those reported in \cite{lim2018forecasting,bica2020estimating}. The aforementioned papers calculate the test RMSE based on both counterfactual trajectories after rolling origin \textbf{and} historical factual trajectories \emph{before} rolling origin. However, the latter biases evaluation towards factual performance. For that reason, we opted for a more challenging evaluation that directly matches our aim, namely predicting counterfactuals over time. Therefore, we only measure performance with respect to counterfactual outcomes after rolling origin (and thus without considering historical factual patient trajectories). 

\subsection{Additional Results}

In the following, we provide additional results for one-step-ahead prediction (Table~\ref{tab:tg-sim-one-step}), $\tau$-step-ahead prediction in a setting with single sliding treatment (Table~\ref{tab:tg-sim-tau-step-timing}), and $\tau$-step-ahead prediction with random trajectories (Table~\ref{tab:tg-sim-tau-step-random}). Note that \shortname ($\alpha = 0$) uses the same model and hyperparameters as \shortname. The only difference is that we switched off our CDC loss. 

In the setting of random trajectories (Table~\ref{tab:tg-sim-tau-step-random}), RMSEs become lower for increasing projection horizons. This can be expected as the application of treatment should decrease the tumor volume. This results in a lower error of estimation. Importantly, the results confirm the superiority of our \longname.

\begin{table*}[h!]
    \caption{Normalized RMSE for one-step-ahead prediction. Shown: mean and standard deviation over five runs (lower is better). Parameter $\gamma$ is the the amount of time-varying confounding: higher values mean larger treatment assignment bias.}
    \label{tab:tg-sim-one-step}
    \begin{center}
        \begin{small}
        \begin{tabular}{l|ccccc}
\toprule
{} &            $\gamma = 0$ &            $\gamma = 1$ &            $\gamma = 2$ &            $\gamma = 3$ &            $\gamma = 4$ \\
\midrule
MSMs                      &           1.107 ± 0.113 &           1.222 ± 0.108 &           1.410 ± 0.089 &           1.680 ± 0.118 &           2.023 ± 0.230 \\
RMSNs                     &           1.037 ± 0.123 &           1.104 ± 0.116 &           1.124 ± 0.115 &           1.268 ± 0.116 &           1.399 ± 0.196 \\
CRN                      &           0.782 ± 0.053 &           0.817 ± 0.050 &           0.887 ± 0.072 &           1.063 ± 0.124 &           1.301 ± 0.144 \\
G-Net                    &           0.832 ± 0.052 &           0.873 ± 0.080 &           1.000 ± 0.062 &           1.299 ± 0.303 &           1.375 ± 0.250 \\
\midrule
CT ($\alpha = 0$) (ours) &           0.778 ± 0.065 &  \textbf{0.790 ± 0.081} &           0.869 ± 0.075 &  \textbf{1.024 ± 0.148} &  \textbf{1.300 ± 0.220} \\
CT (ours)                &  \textbf{0.775 ± 0.063} &           0.797 ± 0.066 &  \textbf{0.859 ± 0.070} &           1.046 ± 0.147 &           1.316 ± 0.229 \\
\bottomrule
\multicolumn{6}{l}{Lower $=$ better (best in bold). }
\end{tabular}

        \end{small}
    \end{center}
    \vskip -0.1in
\end{table*}

\begin{table*}[h!]
    \caption{Normalized RMSE for $\tau$-step-ahead prediction (here: random trajectories setting). Shown: mean and standard deviation over five runs (lower is better). Parameter $\gamma$ is the the amount of time-varying confounding: higher values mean larger treatment assignment bias.}
    \label{tab:tg-sim-tau-step-random}
    \begin{center}
        \begin{small}
        \begin{tabular}{ll|ccccc}
\toprule
           &           &          $\gamma = 0$ &          $\gamma = 1$ &          $\gamma = 2$ &          $\gamma = 3$ &          $\gamma = 4$ \\
\midrule
$\tau = 2$ & MSMs &           1.04 ± 0.04 &           1.21 ± 0.13 &           1.50 ± 0.23 &           1.73 ± 0.43 &           1.85 ± 0.71 \\
           & RMSNs &           1.01 ± 0.09 &           1.03 ± 0.12 &           1.00 ± 0.10 &           1.13 ± 0.16 &           1.09 ± 0.22 \\
           & CRN &           0.77 ± 0.04 &  \textbf{0.76 ± 0.05} &  \textbf{0.81 ± 0.07} &           0.94 ± 0.13 &           1.12 ± 0.25 \\
\vspace{0.9mm}
           & G-Net &           0.94 ± 0.13 &           0.95 ± 0.09 &           1.01 ± 0.05 &           1.10 ± 0.13 &           1.20 ± 0.26 \\
           & CT ($\alpha = 0$) (ours) &           0.76 ± 0.06 &  \textbf{0.76 ± 0.05} &           0.82 ± 0.07 &           0.92 ± 0.21 &           1.09 ± 0.28 \\
           & CT (ours) &  \textbf{0.75 ± 0.06} &           0.77 ± 0.06 &  \textbf{0.81 ± 0.08} &  \textbf{0.90 ± 0.18} &  \textbf{1.06 ± 0.27} \\
\midrule
$\tau = 3$ & MSMs &           1.00 ± 0.04 &           1.14 ± 0.12 &           1.38 ± 0.22 &           1.54 ± 0.38 &           1.51 ± 0.59 \\
           & RMSNs &           0.96 ± 0.05 &           1.02 ± 0.09 &           0.98 ± 0.10 &           1.11 ± 0.20 &           1.17 ± 0.29 \\
           & CRN &           0.78 ± 0.03 &  \textbf{0.78 ± 0.06} &  \textbf{0.83 ± 0.09} &           1.05 ± 0.16 &           1.23 ± 0.32 \\
\vspace{0.9mm}
           & G-Net &           1.01 ± 0.15 &           1.03 ± 0.12 &           1.07 ± 0.07 &           1.15 ± 0.20 &           1.35 ± 0.32 \\
           & CT ($\alpha = 0$) (ours) &           0.76 ± 0.04 &  \textbf{0.78 ± 0.06} &  \textbf{0.83 ± 0.10} &           0.95 ± 0.25 &           1.16 ± 0.37 \\
           & CT (ours) &  \textbf{0.75 ± 0.04} &           0.79 ± 0.06 &  \textbf{0.83 ± 0.11} &  \textbf{0.93 ± 0.23} &  \textbf{1.12 ± 0.32} \\
\midrule
$\tau = 4$ & MSMs &           0.90 ± 0.06 &           1.02 ± 0.11 &           1.22 ± 0.21 &           1.31 ± 0.31 &           1.25 ± 0.51 \\
           & RMSNs &           0.89 ± 0.06 &           0.98 ± 0.08 &           0.92 ± 0.10 &           1.06 ± 0.22 &           1.15 ± 0.31 \\
           & CRN &           0.74 ± 0.03 &  \textbf{0.74 ± 0.07} &           0.80 ± 0.10 &           1.07 ± 0.17 &           1.23 ± 0.34 \\
\vspace{0.9mm}
           & G-Net &           0.97 ± 0.15 &           0.97 ± 0.13 &           1.01 ± 0.08 &           1.07 ± 0.21 &           1.33 ± 0.34 \\
           & CT ($\alpha = 0$) (ours) &           0.72 ± 0.03 &           0.75 ± 0.06 &  \textbf{0.79 ± 0.11} &           0.93 ± 0.28 &           1.14 ± 0.39 \\
           & CT (ours) &  \textbf{0.71 ± 0.03} &           0.75 ± 0.06 &           0.80 ± 0.12 &  \textbf{0.90 ± 0.26} &  \textbf{1.07 ± 0.35} \\
\midrule
$\tau = 5$ & MSMs &           0.80 ± 0.07 &           0.89 ± 0.10 &           1.06 ± 0.20 &           1.10 ± 0.27 &           1.08 ± 0.47 \\
           & RMSNs &           0.81 ± 0.06 &           0.93 ± 0.07 &           0.85 ± 0.10 &           0.99 ± 0.22 &           1.09 ± 0.30 \\
           & CRN &           0.68 ± 0.04 &  \textbf{0.68 ± 0.07} &           0.75 ± 0.10 &           1.03 ± 0.16 &           1.17 ± 0.34 \\
\vspace{0.9mm}
           & G-Net &           0.88 ± 0.14 &           0.88 ± 0.14 &           0.92 ± 0.08 &           0.97 ± 0.21 &           1.26 ± 0.36 \\
           & CT ($\alpha = 0$) (ours) &  \textbf{0.66 ± 0.03} &           0.69 ± 0.06 &  \textbf{0.73 ± 0.11} &           0.88 ± 0.29 &           1.08 ± 0.38 \\
           & CT (ours) &  \textbf{0.66 ± 0.03} &           0.70 ± 0.06 &           0.74 ± 0.12 &  \textbf{0.84 ± 0.26} &  \textbf{1.01 ± 0.34} \\
\midrule
$\tau = 6$ & MSMs &           0.71 ± 0.07 &           0.78 ± 0.09 &           0.91 ± 0.18 &           0.93 ± 0.23 &           0.99 ± 0.44 \\
           & RMSNs &           0.73 ± 0.05 &           0.87 ± 0.06 &           0.77 ± 0.09 &           0.90 ± 0.21 &           1.01 ± 0.28 \\
           & CRN &           0.62 ± 0.04 &  \textbf{0.62 ± 0.07} &           0.70 ± 0.09 &           0.96 ± 0.15 &           1.10 ± 0.32 \\
\vspace{0.9mm}
           & G-Net &           0.79 ± 0.12 &           0.79 ± 0.13 &           0.82 ± 0.09 &           0.86 ± 0.20 &           1.18 ± 0.35 \\
           & CT ($\alpha = 0$) (ours) &  \textbf{0.59 ± 0.02} &           0.63 ± 0.06 &  \textbf{0.67 ± 0.11} &           0.80 ± 0.29 &           1.00 ± 0.36 \\
           & CT (ours) &  \textbf{0.59 ± 0.02} &           0.63 ± 0.06 &  \textbf{0.67 ± 0.12} &  \textbf{0.77 ± 0.25} &  \textbf{0.93 ± 0.32} \\
\bottomrule
\multicolumn{3}{l}{Lower $=$ better (best in bold). }
\end{tabular}

        \end{small}
    \end{center}
    \vskip -0.1in
\end{table*}

\begin{table*}[h!]
    \caption{Normalized RMSE for $\tau$-step-ahead prediction (here: single sliding treatment setting). Shown: mean and standard deviation over five runs (lower is better). Parameter $\gamma$ is the the amount of time-varying confounding: higher values mean larger treatment assignment bias.}
    \label{tab:tg-sim-tau-step-timing}
    \begin{center}
        \begin{small}
        \begin{tabular}{ll|ccccc}
\toprule
           &           &          $\gamma = 0$ &          $\gamma = 1$ &          $\gamma = 2$ &          $\gamma = 3$ &          $\gamma = 4$ \\
\midrule
$\tau = 2$ & MSMs &           1.33 ± 0.13 &           1.59 ± 0.20 &           1.88 ± 0.36 &           2.23 ± 0.63 &           2.51 ± 0.91 \\
           & RMSNs &           0.98 ± 0.12 &           1.10 ± 0.25 &           0.98 ± 0.08 &           1.18 ± 0.10 &  \textbf{0.94 ± 0.09} \\
           & CRN &           0.71 ± 0.07 &           0.75 ± 0.06 &           0.77 ± 0.04 &           0.94 ± 0.14 &           1.11 ± 0.22 \\
\vspace{0.9mm}
           & G-Net &           0.99 ± 0.16 &           0.99 ± 0.06 &           1.03 ± 0.09 &           1.10 ± 0.08 &           1.18 ± 0.16 \\
           & CT ($\alpha = 0$) (ours) &  \textbf{0.70 ± 0.09} &  \textbf{0.71 ± 0.09} &           0.76 ± 0.08 &  \textbf{0.90 ± 0.21} &           1.00 ± 0.21 \\
           & CT (ours) &  \textbf{0.70 ± 0.09} &           0.72 ± 0.09 &  \textbf{0.74 ± 0.07} &  \textbf{0.90 ± 0.13} &           1.01 ± 0.23 \\
\midrule
$\tau = 3$ & MSMs &           1.61 ± 0.15 &           1.90 ± 0.24 &           2.20 ± 0.42 &           2.53 ± 0.72 &           2.64 ± 0.95 \\
           & RMSNs &           0.98 ± 0.10 &           1.16 ± 0.21 &           1.00 ± 0.09 &           1.23 ± 0.12 &  \textbf{1.06 ± 0.14} \\
           & CRN &  \textbf{0.73 ± 0.06} &           0.78 ± 0.06 &           0.85 ± 0.06 &           1.16 ± 0.26 &           1.34 ± 0.37 \\
\vspace{0.9mm}
           & G-Net &           1.15 ± 0.20 &           1.16 ± 0.11 &           1.20 ± 0.15 &           1.24 ± 0.12 &           1.47 ± 0.22 \\
           & CT ($\alpha = 0$) (ours) &  \textbf{0.73 ± 0.08} &  \textbf{0.75 ± 0.08} &           0.82 ± 0.09 &           0.99 ± 0.25 &           1.13 ± 0.28 \\
           & CT (ours) &  \textbf{0.73 ± 0.08} &           0.76 ± 0.07 &  \textbf{0.79 ± 0.08} &  \textbf{0.98 ± 0.19} &           1.12 ± 0.27 \\
\midrule
$\tau = 4$ & MSMs &           1.79 ± 0.16 &           2.08 ± 0.26 &           2.37 ± 0.45 &           2.67 ± 0.76 &           2.62 ± 0.94 \\
           & RMSNs &           0.99 ± 0.10 &           1.18 ± 0.17 &           1.03 ± 0.11 &           1.28 ± 0.15 &  \textbf{1.21 ± 0.19} \\
           & CRN &  \textbf{0.76 ± 0.05} &           0.81 ± 0.07 &           0.93 ± 0.08 &           1.35 ± 0.38 &           1.55 ± 0.50 \\
\vspace{0.9mm}
           & G-Net &           1.25 ± 0.24 &           1.24 ± 0.14 &           1.27 ± 0.21 &           1.29 ± 0.14 &           1.64 ± 0.28 \\
           & CT ($\alpha = 0$) (ours) &  \textbf{0.76 ± 0.07} &  \textbf{0.79 ± 0.06} &           0.87 ± 0.11 &           1.06 ± 0.27 &  \textbf{1.21 ± 0.32} \\
           & CT (ours) &  \textbf{0.76 ± 0.07} &           0.80 ± 0.06 &  \textbf{0.85 ± 0.09} &  \textbf{1.05 ± 0.22} &  \textbf{1.21 ± 0.30} \\
\midrule
$\tau = 5$ & MSMs &           1.88 ± 0.17 &           2.15 ± 0.27 &           2.42 ± 0.45 &           2.69 ± 0.75 &           2.54 ± 0.90 \\
           & RMSNs &           1.00 ± 0.10 &           1.19 ± 0.14 &           1.08 ± 0.13 &           1.34 ± 0.19 &           1.39 ± 0.31 \\
           & CRN &  \textbf{0.79 ± 0.04} &           0.85 ± 0.07 &           1.01 ± 0.11 &           1.51 ± 0.47 &           1.72 ± 0.58 \\
\vspace{0.9mm}
           & G-Net &           1.29 ± 0.26 &           1.28 ± 0.18 &           1.32 ± 0.25 &           1.33 ± 0.15 &           1.76 ± 0.37 \\
           & CT ($\alpha = 0$) (ours) &  \textbf{0.79 ± 0.06} &  \textbf{0.83 ± 0.06} &           0.92 ± 0.12 &           1.12 ± 0.30 &           1.28 ± 0.34 \\
           & CT (ours) &  \textbf{0.79 ± 0.07} &           0.84 ± 0.07 &  \textbf{0.89 ± 0.11} &  \textbf{1.11 ± 0.24} &  \textbf{1.26 ± 0.31} \\
\midrule
$\tau = 6$ & MSMs &           1.89 ± 0.17 &           2.14 ± 0.26 &           2.39 ± 0.44 &           2.62 ± 0.73 &           2.41 ± 0.85 \\
           & RMSNs &           1.03 ± 0.10 &           1.21 ± 0.12 &           1.12 ± 0.14 &           1.41 ± 0.25 &           1.58 ± 0.45 \\
           & CRN &  \textbf{0.82 ± 0.04} &           0.89 ± 0.07 &           1.08 ± 0.13 &           1.64 ± 0.54 &           1.83 ± 0.62 \\
\vspace{0.9mm}
           & G-Net &           1.33 ± 0.27 &           1.31 ± 0.22 &           1.35 ± 0.29 &           1.35 ± 0.16 &           1.86 ± 0.47 \\
           & CT ($\alpha = 0$) (ours) &  \textbf{0.82 ± 0.04} &  \textbf{0.86 ± 0.05} &           0.96 ± 0.12 &           1.19 ± 0.33 &           1.32 ± 0.34 \\
           & CT (ours) &  \textbf{0.82 ± 0.05} &           0.88 ± 0.06 &  \textbf{0.93 ± 0.11} &  \textbf{1.16 ± 0.25} &  \textbf{1.29 ± 0.29} \\
\bottomrule
\multicolumn{7}{l}{Lower $=$ better (best in bold). }
\end{tabular}

        \end{small}
    \end{center}
    \vskip -0.1in
\end{table*}

\clearpage
\section{Details on Experiments with Semi-Synthetic Data} \label{app:ss-sim}

\subsection{Data} 

We used the MIMIC-extract \cite{wang2020mimic} with a standardized preprocessing pipeline of MIMIC-III dataset \cite{johnson2016mimic}. MIMIC-extract provides intensive care unit~(ICU) data aggregated at hourly levels. We used forward and backward filling for missing values and did standard normalization of all the continuous time-varying features.

For our semi-synthetic data, we then extract 25 different vital signs (as time-varying covariates) and 3 static covariates, \ie, gender, ethnicity, and age (as static covariates). The full list of features is given in the code of our GitHub repository for reproducibility. We one-hot-encode all static covariates (gender, ethnicity, and age) and use them later further for generating noise. Altogether, this results into a 44-dimensional feature vector ($d_v = 44$).

Our simulation of semi-synthetic data extends the basic idea of \cite{schulam2017reliable}. As such, we first generate untreated trajectories of outcomes under endogeneous and exogeneous dependencies and, then, sequentially apply treatments to the trajectory. Dependencies between treatments, outcomes, and time-varying covariates are assumed to be sparse, so outcomes are influenced by only a few treatments and time-varying covariates. Treatment assignment in turn depends on a few outcomes and time-varying covariates. 

Our semi-synthetic simulator proceeds as follows. \emph{First}, we select a cohort of 1,000 patients, which are randomly chosen from patients where the intensive care unit stay lasted at least 20 hours. For the selected cohort, we clip intensive care unit stays longer than 100 hours (so that $T^{(i)}$ ranges from 20 to 100). 

\emph{Second}, we simulate $d_{y}$ untreated outcomes $\mathbf{Z}_{t}^{j, (i)}, j = 1 \dots, d_{y}$, for each patient $i$ from the cohort. Therein, we combine (1)~an endogenous component ($\text{B-spline}(t)$ and random function $g^{j, (i)}(t)$); (2)~an exogenous dependency $f_Z^j(\mathbf{X}_{t}^{(i)})$ on a subset of current time-varying covariates; and (3)~independent random noise $\varepsilon_t$. Formally, we generate the simulations via
\begin{align}
    \mathbf{Z}_{t}^{j, (i)} = \underbrace{\alpha_S^j \, \text{B-spline}(t) + \alpha_g^j \, g^{j, (i)}(t)}_\text{endogenous} +  \underbrace{\alpha_f^j \, f_Z^j(\mathbf{X}_{t}^{(i)})}_{\text{exogenous}} + \underbrace{\varepsilon_t}_{\text{noise}} 
\end{align}
with $\varepsilon_t \sim N(0, 0.005^2)$ and where $\alpha_S^j$, $\alpha_g^j$, and $\alpha_f^j$ are weight parameters. Further, $\text{B-spline}(t)$ is sampled from a mixture of three cubic splines (one with rapid decline during all intensive care unit stay, one with a mild decline, and one stable); $g^{j, (i)}(\cdot)$ is sampled independently for each patient from Gaussian process with Matérn kernel; and $f_Z^j(\cdot)$ is sampled from a random Fourier features (RFF) approximation of an Gaussian process \cite{hensman2017variational}. Here, we specifically use RFF as they circumvent the need for tedious Cholesky decomposition when sampling random functions at many points of time-varying feature space $\mathbb{R}^{d_x}$. By combining all three components, we aim to simulate outcomes, which have endogeneous dependencies with different resolutions (global trends of B-splines and local correlation structure of Gaussian processes) and arbitrarily chosen exogeneous dependencies on other time-varying features. 

\emph{Third}, we sequentially simulate synthetic $d_a$ binary treatments $\mathbf{A}_{t}^{l}$, $l = 1, \dots, d_a$. We add confounding to the treatments by a subset of current time-varying covariates via a random function $f_Y^l(\mathbf{X}_{t})$. Subsequently, we average of the subset of previous $T_l$ treated outcomes $\bar{A}_{T_l}(\bar{\mathbf{Y}}_{t-1})$. Formally, we compute $\mathbf{A}_{t}^{l}$ via
\begin{align}
    p_{\mathbf{A}_{t}^{l}} & =  \sigma \Big( \gamma_A^l \bar{A}_{T_l}( \bar{\mathbf{Y}}_{t-1}) + \gamma_X^l f_Y^l(\mathbf{X}_{t}) +  b_l \Big), \\
    \mathbf{A}_{t}^{l} & \sim \operatorname{Bernoulli}\big( p_{\mathbf{A}_{t}^{l}} \big) ,
\end{align}
where $\sigma(\cdot)$ is the sigmoid activation, $ \gamma_A^l$ and $\gamma_X^l$ are confounding parameters, $b_l$ is a fixed bias, and $f_Y^l(\cdot)$ is sampled from an RFF approximation of a Gaussian process (similar to $f_Z^j(\cdot)$).

\emph{Fourth}, we apply treatments to the untreated outcomes. For this, we set $\mathbf{Y}_1 = \mathbf{Z}_1$. Each treatment $l$ is modeled so that it has a long-lasting effect on some outcome $j$, with maximal additive effect $\beta_{lj}$ right after application. Here, we assume that the treatment has an effect within a time window $t - w^l, \dots, t$. We further assume that the effect size of treatments is subject to an inverse-square decay over time. We also scale the effect by the probability $p_{\mathbf{A}_{t}^{l}}$. Afterward, the effects of multiple treatments are aggregated by taking the minimum across the treatment effects. Formally, we model this via
\begin{align}
    E^{j}(t) & = \sum_{i=t-w^l}^{t} \frac{\min_{l=1,\dots, d_a} \mathbbm{1}_{[\mathbf{A}_{i}^{l} = 1]} p_{\mathbf{A}_{i}^{l}} \beta_{lj}}{(w^l - i)^2} ,
\end{align} 
where $\beta_{lj}$ is the maximum effect size of treatment $l$. This is either constant for all the outcomes $j$, or zero, so that the treatment does not influence the outcome.   

\emph{Fifth}, we combine the above. That is, we simply add the simulated treatment effect $E^{j}(t)$ to untreated outcome; \ie,
\begin{align}
    \mathbf{Y}_{t}^j & = \mathbf{Z}_{t}^j + E^{j}(t) .
\end{align}

\emph{Sixth}, we generate our semi-synthetic dataset based on the above simulator. For exact values of all simulation parameters, we refer to code implementation. After simulating three synthetic binary treatments ($d_a = 3$) and two synthetic outcomes ($d_y = 2$), we split the cohort of 1,000 patients into train/validation/test subsets via a 60\% / 20\% / 20 \% split. For one-step-ahead prediction, we then simulate all $2^3 = 8$ counterfactual outcomes. For multiple-step-ahead prediction with $\tau_{\text{max}} = 10$, we sample 10 random trajectories for each patient/time step. 

\subsection{Experimental Details}

\textbf{Hyperparameter tuning.} We perform hyperparameter tuning separately for all models. For this, we use the 200 factual patient time-series from the validation subset. Details are in Appendix~\ref{app:hparams}.

\textbf{Performance measurement:} We retrain the models on five simulated datasets with different random seeds (random seeds for sampling from Gaussian processes are kept the same). We then report the averaged root mean square error (RMSE) on the test set, that is, for hold-out data. RMSE is calculated for standardized outcomes.

\clearpage
\section{Details on Experiments with Real-World Data} \label{app:real-world-data}

\subsection{Data} 

Similarly to the semi-synthetic data in Appendix~\ref{app:ss-sim}, we make use of the MIMIC-extract following a standardized preprocessing pipeline \cite{wang2020mimic}. The data gives measurements from intensive care units aggregated at hourly levels. We used forward and backward filling for missing values and did standard normalization of all the continuous time-varying features.

We use the same 25 vital signs ($d_x = 25$) and the 3 static features (also one-hot-encoded for categorical features, $d_v = 44$) as in the semi-synthetic experiments. Both time-varying covariates and static features serve as potential confounders. We use two binary treatments ($d_a = 2$): vasopressors and mechanical ventilation. We then estimate the factual outcome ($d_y = 1$): (diastolic) blood pressure. Here, it is known that this may be positively or negatively affected by vasopressors and mechanical ventilation, thus raising the question for clinical practitioners of how a patient trajectory may look like when such treatment is applied. 

For our experiments, we selected a cohort of 5,000 patients, randomly chosen from the patients with intensive care unit~(ICU) stays of at least 30 hours. For the selected cohort, we cut off ICU stays at 60 hours. We then split the cohort of 5,000 patients with a ratio of 70\%/15\%/15\% into train/validation/test subsets. We varied the implementation according to the projection horizon $\tau$. (i)~For the one-step-ahead prediction, we used all trajectories in the test set. (ii)~For a $\tau$-step-ahead prediction with $\tau \ge 2$, we proceed as follows. Let $\tau_{\text{max}} \ge \tau$ denote the maximum projection horizon. In our experiments, $\tau_{\text{max}} = 5$. We then extract all sub-trajectories with a length of at least $\tau_{\text{max}} + 1$ with a rolling origin, where we remove vital signs from time steps $1, \dots, T^{(i)} - \tau_{\text{max}} + 1$, respectively. We then make predictions but where a looking-ahead is prevented due to masking. Later, we report only the performance for the $\tau$-step-ahead prediction. 

\subsection{Experimental Details}

\textbf{Hyperparameter tuning.} We perform hyperparameter tuning separately for all models. For this, we use the 750 factual patient time-series from the validation subset. Details are in Appendix~\ref{app:hparams}.

\textbf{Performance measurement:} We retrain the models on five random sub-samples of the dataset with different random seeds. We then report the averaged root mean square error (RMSE) on the test set, that is, for hold-out data. RMSE is then unscaled to the original range with standard normalization parameters.

\clearpage
\section{Runtime and Model Size Comparison} \label{app:runtime}
\shortname with a single-stage training also provides a decent speed up for training and inference in comparison to other methods. In Table~\ref{tab:tg-sim-runtime}, we compare the total runtime of experiments, averaged over all confounding levels $\gamma = 0, \dots, 4$, for synthetic data. Among all neural models, our \shortname has the smallest runtime. Hence, our transformer architecture together with a single-stage training procedure with CDC loss not only improves the performance of counterfactual outcomes estimations but also achieves a substantial computational speed-up. In Table~\ref{tab:num_params}, we report the total number of trainable parameters for different models after hyperparameter tuning. For semi-synthetic and real-world data, \shortname turns out to be more parsimonious, than LSTM-based models.

\begin{table*}[h!]
    \caption{Runtime of experiments (all stages of training and inference) for both tasks of one- and $\tau$-step-ahead prediction, averaged over different $\gamma = 0, \dots, 4$ (lower is better). Total runtime includes data generation. Experiments are carried out on $1\times$ TITAN V GPU.}
    \label{tab:tg-sim-runtime}
    \begin{center}
        \begin{small}
        \begin{tabular}{l|l|rcl}
            \toprule
                    & Main stages of training \& inference & \multicolumn{3}{c}{Total runtime (in min)} \\
            \midrule
             MSMs   & 2 logistic regressions for IPTW \& linear regression & 3.5 &±& 0.3  \\
             RMSNs  & 2 networks for IPTW \& encoder \& decoder & 109.7 &±& 2.3 \\
             CRN    & encoder \& decoder & 75.3 &±& 17.5 \\
             G-Net  & single network \& MC sampling for inference & 118.0 &±& 2.0 \\
             \midrule
             CT (ours)  & single multi-input network & \textbf{13.5} &±& \textbf{4.8} \\
             \bottomrule
        \end{tabular}
        \end{small}
    \end{center}
    \vskip -0.1in
\end{table*}

\begin{table*}[h!]
    \centering
    
    \caption{Total number of trainable parameters of models after hyperparameter tuning. Here, we distinguish (1)~data using the tumor growth (TG) simulator ($=$experiments with fully-synthetic data), (2)~data from semi-synthetic (SS) benchmark, and (3)~real-world (RW) MIMIC-III data. The number is the sum of trainable parameters among all the sub-models for MSMs, RMSNs, and CRN.}
    \label{tab:num_params}
    \vskip 0.1in
    \begin{small}
    \begin{tabular}{l|ccccc|c|c}
        \toprule
                    & \multicolumn{5}{c|}{TG simulator} & \multirow{2}{*}{SS data} & \multirow{2}{*}{RW data}\\
                    & $\gamma = 0$ & $\gamma = 1$ & $\gamma = 2$ & $\gamma = 3$ & $\gamma = 4$ & & \\
        \midrule
         MSMs       & \multicolumn{5}{c|}{$<$100} & 3K & 1K \\
         RMSNs      & 20K   & 4K    & 23K   & 21K   & 22K   & 477K  & 947K \\
         CRN        & 4K    & 6K    & 8K    & 7K    & 8K    & 165K  & 219K \\
         G-Net      &  3K    & 2K   & 3K    & 4K    & 3K    & 151K  & 310K \\
         \midrule
         CT (ours)  &  11K   & 11K  & 10K   & 10K   & 10K   & 45K   & 69K   \\
         \bottomrule
    \end{tabular}
    \end{small}
    \vskip -0.1in
\end{table*}

\clearpage
\section{Visualization of Learned Representations} 
\label{app:t-sne}
Figure~\ref{fig:tsne}(a,b) visualizes the t-SNE embeddings for the balanced representations of \shortname. Here, we use the validation set of the fully-synthetic data from the tumor growth simulator. Colors show the health outcome (tumor volume). As the plots show, we observe several regions  where representations are indeed balanced, so that they appear non-predictive of the current treatment but expressive of the outcome. To this end, one can observe a continuous change in color (outcome). In severe cases, the points are colored in yellow when tumor size is comparatively large. As we can see, balancing then becomes challenging, as few patients with this condition receive no treatment.

\begin{figure}[h!]
    \centering
    \hfill
    \subfigure[t-SNE embeddings of balanced representations with indicated current treatments ]{\includegraphics[width=0.4\textwidth]{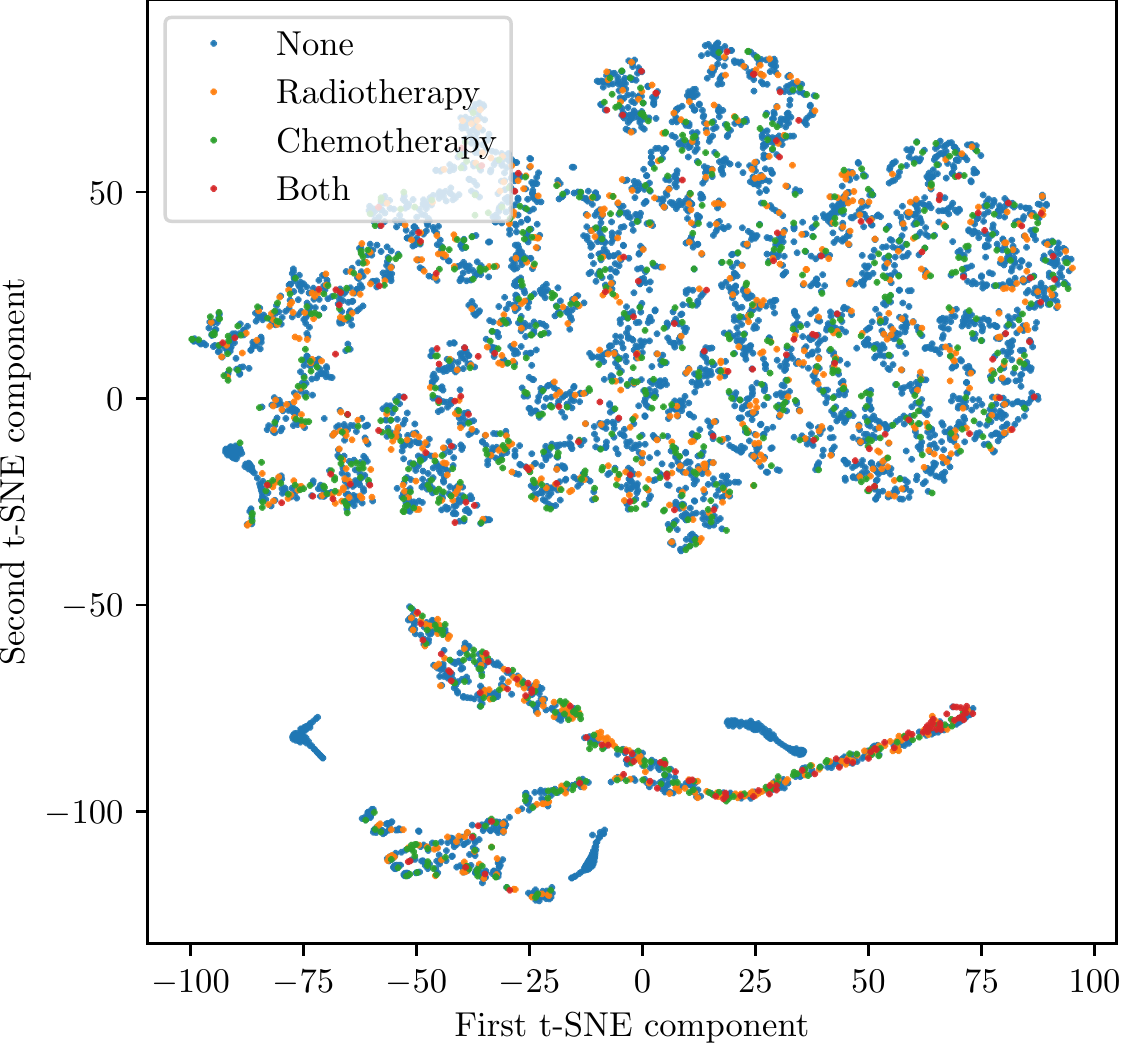}}\label{fig:tsne-treatments}
    \hfill
    \subfigure[t-SNE embeddings of balanced representations with indicated next outcomes]{\includegraphics[width=0.48\textwidth]{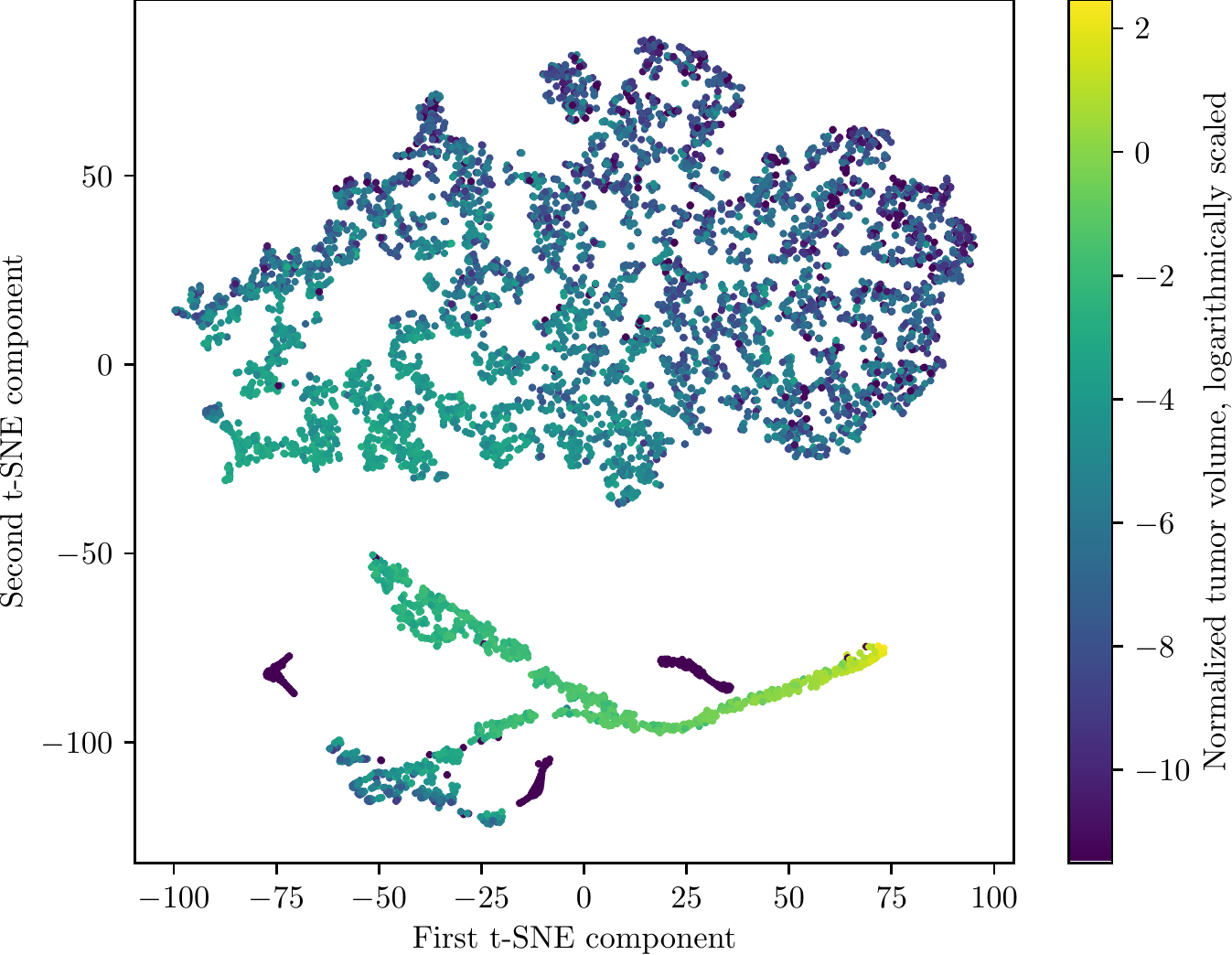}}\label{fig:tsne-outcomes}
    \hfill
    \caption{t-SNE embeddings of the balanced representations of \shortname. We display $N = 100$ patients from the fully-synthetic data (tumor growth simulator). Here: representations of the validation set ($\gamma=4$), where each patient trajectory contains 60 time steps, thus displaying 6,000 embeddings. Note the logarithmic scale for the outcomes (in color).}
    \label{fig:tsne}
    \vskip -0.2in
\end{figure}

\clearpage
\section{Case Study: Importance of Subnetworks} \label{app:subnetwork-importance}

In the following, we provide a case study for explainability. That is, we study the importance of different subnetworks. Such insights may help medical practitioners to ponder about the relevance of treatments for patient outcomes, or the relevance of time-varying covariates for patient outcomes. 

To this end, we examine the role of using multiple cross-attention for information exchange between three subnetworks of \shortname. We informally define an importance score of subnetwork $A$, $Y$, or $X$ as the difference in performance (\eg, with test RMSE) between full \shortname and \shortname with the correspondingly isolated subnetwork. Here, isolating a subnetwork means that we remove cross-attentions of the particular subnetwork. As such, it does not completely ignore the input sequence but only the interactions, as we still use sequences of all subnetworks representations at the latest stage of average pooling. Therefore, the importance score aims to explain how the connectivity of subnetworks via cross-attentions helps in estimating counterfactuals over time.

For our case study, we use the semi-synthetic benchmark and kept the same hyperparameters, as for the original \shortname. Figure~\ref{fig:subnetwork-importance} shows importance scores of each subnetwork for different prediction horizons $\tau$. We observe that subnetwork processing time-varying covariates has the largest importance score (red bars). Interestingly, the importance score for the treatment subnetwork is close to zero for a small prediction horizon $\tau$ and grows only for larger prediction horizons. This thus has implications: it suggests \emph{long-range} treatment effects.

\begin{figure}[h]
    \centering
    \includegraphics[width=\textwidth]{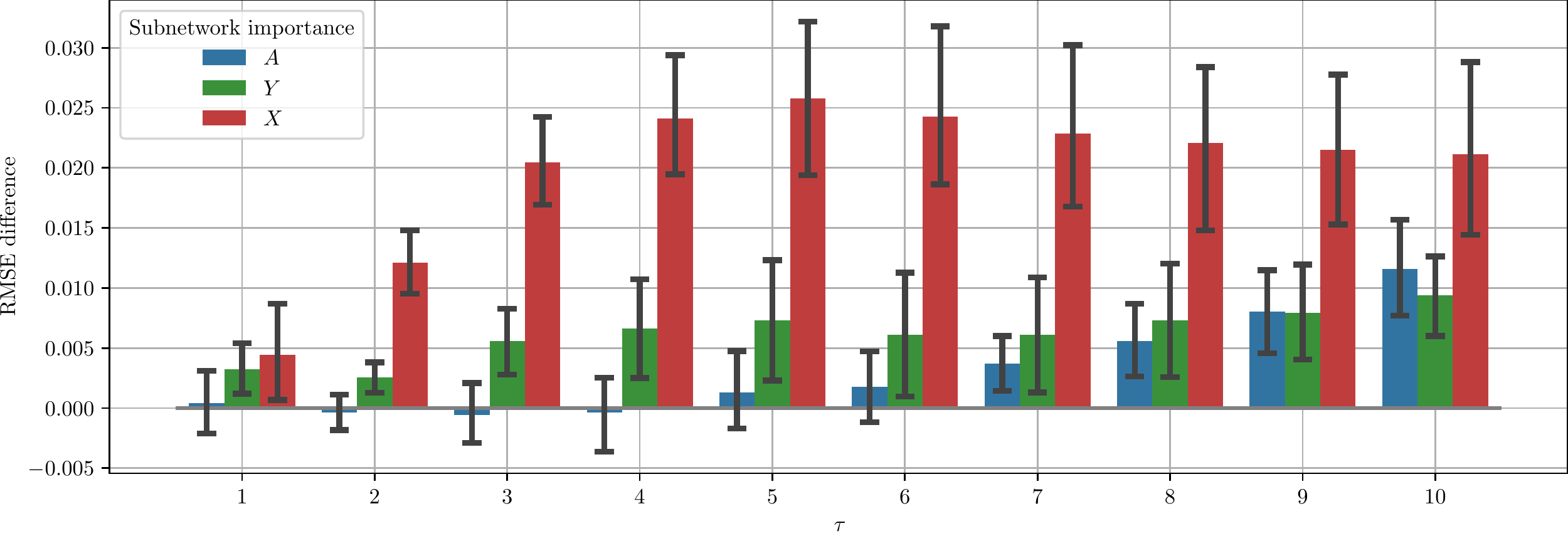}
    \vskip -0.1in
    \caption{Subnetworks importance scores based on semi-synthetic benchmark (higher values correspond to higher importance of subnetwork connectivity via cross-attentions). Shown: RMSE differences between model with isolated subnetwork and full \shortname, means $\pm$ standard errors.}
    \label{fig:subnetwork-importance}
\end{figure}

\end{document}